\documentclass{article}

\IfFileExists{iclr2027_conference.sty}{%
  \usepackage{iclr2027_conference,times}%
}{%
  \usepackage{times}
  \usepackage[letterpaper,left=1in,right=1in,top=1in,bottom=1in]{geometry}
  
  \newcommand{\iclrfinalcopy}{}
  \usepackage[round]{natbib}
  \makeatletter
  \renewcommand{\maketitle}{%
    \begin{center}
      {\LARGE\bfseries \@title \par}\vspace{1.2em}
      {\normalsize \@author \par}\vspace{0.9em}
    \end{center}}
  \makeatother
}

\usepackage[T1]{fontenc}   
\usepackage{amsmath,amssymb,amsthm}
\usepackage{booktabs}
\usepackage{graphicx}
\usepackage{xcolor}
\usepackage{tikz}
\usetikzlibrary{arrows.meta,positioning,calc}
\usepackage[colorlinks=true,linkcolor=blue!55!black,citecolor=blue!55!black,
            urlcolor=blue!55!black]{hyperref}
\usepackage{url}
\usepackage{multirow}
\usepackage{tabularx}
\usepackage{array}
\usepackage{caption}
\usepackage{placeins}   

\newtheorem{proposition}{Proposition}
\theoremstyle{remark}

\iclrfinalcopy

\definecolor{good}{HTML}{2E7D32}
\definecolor{bad}{HTML}{B71C1C}
\definecolor{acc}{HTML}{E07B39}
\definecolor{navy}{HTML}{1F4E79}
\newcommand{\pass}{\textcolor{good}{\textbf{pass}}}

\newcommand{\pid}[1]{\textcolor{black!45}{\,(#1)}}

\renewcommand{\pid}[1]{}   

\newcommand{\dojepa}{\textsc{Do-JEPA}}
\newcommand{\Obs}{\textsc{Obs}}
\newcommand{\Aug}{\textsc{Aug}}
\newcommand{\Paired}{\textsc{Paired}}
\newcommand{\Mask}{\textsc{Mask}}

\newcommand{\anull}{a_{\varnothing}}
\newcommand{\R}{\mathbb{R}}

\newcommand{\dz}{\Delta z}
\newcommand{\dzh}{\widehat{\Delta z}}
\newcommand{\Lpred}{\mathcal{L}_{\mathrm{pred}}}
\newcommand{\Ld}{\mathcal{L}_{\Delta}}
\newcommand{\Lsup}{\mathcal{L}_{\mathrm{sup}}}
\newcommand{\Ledge}{\mathcal{L}_{\mathrm{edge}}}
\newcommand{\Linv}{\mathcal{L}_{\mathrm{inv}}}
\newcommand{\Lctx}{\mathcal{L}_{\mathrm{ctx}}}
\newcommand{\Fb}{F_{\mathrm{base}}}
\newcommand{\Fa}{F_{\mathrm{adapt}}}
\newcommand{\pp}{\,\text{pp}}

\title{\dojepa: From Masking to Intervention \\
       in Latent World Models}

\author{
Hossein Resani$^{1}$ \&
Javen Qinfeng Shi$^{1,2}$\\
$^{1}$Australian Institute for Machine Learning, Adelaide University\\
$^{2}$Responsible AI Research Centre
}

\begin{document}

\maketitle

\begin{abstract}
Latent world models are trained to predict what happens next, so nothing in their
objective separates what an action caused from what merely co-occurred with it.
Object-masking models such as C-JEPA intervene on what the predictor can see; we
intervene on what physically happens. From one saved simulator state we run the
dynamics under an action $a$ and under a reference action $\anull$, and train the
model to predict the difference $\Delta z=z^{a}-z^{\anull}$ between the two latent
futures. The resulting objective, \dojepa, has an effect loss, a support loss
(where the action enters), a propagation loss (where its effect travels) and
invariance losses (what must not change). In a synthetic system with
object-aligned variables, support supervision finds the directly intervened object
in $99.95\%$ of test cases, where a sparse action mask sends the action to a
nuisance slot in every case, and response-onset supervision recovers the
ring-shaped propagation graph (edge AUROC $0.975$ vs.\ $0.624$). From pixels, the
effect loss beats a control trained on exactly the same data: it lowers latent
effect error by $28.4\%$ on an end-to-end LeWM model and physical
effect error by $13.5\%$ when trained and tested on natural action sequences, and on three
independently generated CausalWorld benchmarks it lowers responsive effect error
by about $20\%$ under physics shifts and the latent context sensitivity of
predicted effects by $66\%$. Trained from scratch it costs factual accuracy;
fine-tuning an existing model with it removes this cost. Together, these results show 
that intervening on the world, rather than on what the model sees, helps latent world models
predict what their actions cause.

\end{abstract}

\section{Introduction}\label{sec:intro}

Latent world models learn a transition $\hat z_{t+1}=F_\theta(z_t,a_t)$ in a
learned embedding space and use it to predict and to plan
\citep{ha2018worldmodels,hafner2023dreamerv3,zhou2024dinowm}; joint-embedding
predictive architectures (JEPAs) are a common choice
\citep{lecun2022path,assran2023ijepa,assran2025vjepa2,balestriero2025lejepa}. A
planner asks such a model a counterfactual question: what would happen if I took
this action instead of that one? The model is trained on a different question:
what happened next in the logged data? The two come apart when something the
action does not control co-varies with the outcome, for example a hidden cause
shared by the behaviour policy and the rendered scene. The model can then fit the
logs by crediting the wrong variable, and prediction error does not show it. In
our synthetic benchmark (Sec.~\ref{sec:q12}), the model with the lowest
prediction and effect error has no notion of where an action enters, and a model
with a sparse action-entry mask sends every action into a nuisance slot.

Object-centric JEPAs add structure but keep the observational target. C-JEPA
\citep{nam2026cjepa} masks whole object slots, so that the predictor must infer a
hidden object from the others. As its authors note, this intervenes on
observability, not on the process that generates the data: masking an object
changes what the model sees, not what the world does, and the targets still come
from the logged rollout.

\begin{figure}[t]
\centering
\includegraphics[width=0.82\linewidth]{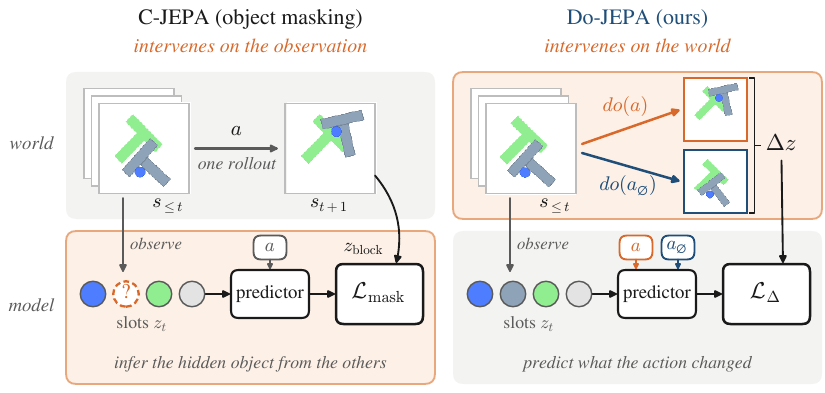}
\caption{\textbf{Masking versus intervention.} Both methods see the same
history, slots and action; orange marks where each intervenes. \textbf{Left:}
C-JEPA \citep{nam2026cjepa} hides an object's slots and predicts them from the
others, with targets from the same logged rollout: it changes what the model sees.
\textbf{Right:} \dojepa\ runs one saved simulator state forward under $a$ and under
$\anull$ and supervises the difference $\Delta z$ of the two latent futures
(Eq.~\ref{eq:delta}): it changes what happens. Frames: Push-T.}
\label{fig:teaser}
\end{figure}

A simulator allows a stronger intervention. We save its full state and run it
forward twice, once with an action $a$ and once with a reference action $\anull$
(Fig.~\ref{fig:teaser}). State, context, history and noise are shared, so the
difference between the two latent futures is the change caused by executing $a$
instead of $\anull$ in that state, and nothing else. We call objectives built on
such pairs \dojepa: C-JEPA intervenes on the observation, \dojepa\ on the world. A
world model that knows what its actions do should say what an action changed
(\emph{effect}), where it entered (\emph{support}), where its effect can travel
(\emph{propagation}) and what must stay the same (\emph{invariance}). \dojepa\ has
one loss for each, computed from the pairs alone. We test each where it can be
measured: support and propagation against known ground truth, invariance with
given variables, and the effect loss end to end from pixels against a control that
sees exactly the same data.


Our contributions are summarised as follows. First, we introduce \dojepa, an
interventional objective for latent world models that derives four training
signals from exact same-state rollout pairs: the effect of an action, where it
enters, where its effect travels and what must not change. It needs no graph
labels or object identities (Sec.~\ref{sec:method}). Second, we show that the direct target and the propagation graph can be recovered in controlled settings. Support supervision identifies the
directly intervened object in $99.95$--$100\%$ of test cases on three seeds,
whereas a sparse mask picks the wrong slot at chance level. Onset supervision
recovers the propagation graph, and the observation rate decides whether one-hop
edges or only reachability can be recovered (Sec.~\ref{sec:q12}). Third, we
provide evidence from pixels against an information-matched control: paired
supervision improves effect prediction on end-to-end LeWM and on three
CausalWorld benchmarks against a control trained on the same data, and
fine-tuning with it kept factual accuracy in our tests (Sec.~\ref{sec:q4}).

\section{Related Work}\label{sec:related}

\paragraph{Latent world models and JEPAs.}
JEPAs predict in embedding space \citep{lecun2022path,assran2023ijepa} and now
drive action-conditioned video models and planners
\citep{assran2025vjepa2,zhou2024dinowm}. LeWM \citep{maes2026lewm} trains such a
model end to end from pixels with the SIGReg regulariser
\citep{balestriero2025lejepa}. We use LeWM as a backbone and, with its CEM
planner \citep{garcia1989mpc,rubinstein1999cem} unchanged, for a planning study
(App.~\ref{app:planning}).

\paragraph{Interventions and causal world models.}
Interventions make latent causal variables identifiable where observations do not
\citep{scholkopf2021crl,lippe2022citris,lippe2023biscuit,brehmer2022weakly,ahuja2023interventional,vonkugelgen2021self};
we share this premise but do not recover causal variables or graphs from pixels.
Other work learns modular mechanisms \citep{lei2022vcd} or interaction graphs
\citep{li2020causaldiscovery}, chooses actions that separate causal hypotheses
\citep{sontakke2021curiosity}, or argues for intervention-centric evaluation
\citep{liu2023causaltriplet}. Closest to us, C-JEPA \citep{nam2026cjepa} masks
object-slot trajectories; we re-implement it on the same slots. 


\section{Method}\label{sec:method}

\subsection{Setting}\label{sec:setting}

\paragraph{Setup and notation.}
An environment has a physical state $s_t$, an action $a_t\in\mathcal A$ and a
context $c\in\mathcal C$ that changes the observation but not the physics: the
next state is $s_{t+1}=f(s_t,a_t,\varepsilon_t)$, with exogenous noise
$\varepsilon_t$, and the observation is $o_t=g(s_t,c)$. Examples of $c$ are a
pattern rendered into the image and a vector of nuisance variables. An encoder
$E_\psi$ maps an observation to $K$ slots, $z_t=(z^1_t,\dots,z^K_t)$ with
$z^i_t\in\R^d$, and a predictor $F_\theta$ maps a history of $H$ latent states
and an action to the next latent state:
\begin{equation}
\hat z_{t+1}=F_\theta\big(z_{t-H+1:t},\,a_t\big)\in\R^{K\times d}.
\label{eq:predictor}
\end{equation}
Object-slot models have one slot per object or nuisance variable ($K>1$); the
LeWM backbone has one global latent ($K=1$, $d=192$). A hat marks a prediction,
$N$ is the batch size, and Table~\ref{tab:notation} lists every symbol with its
shape, role and availability.

\paragraph{JEPA world models.}
A JEPA world model is trained to predict the latent of the next observation,
\begin{equation}
\Lpred=\frac1N\sum_{n=1}^{N}\big\lVert\,\hat z_{t+1,n}-z_{t+1,n}\big\rVert_2^2 ,
\label{eq:pred}
\end{equation}
where $z_{t+1,n}$ encodes the observed next frame of example $n$, together with a
term that prevents collapse (an exponential-moving-average target encoder, or
SIGReg in LeWM). Targets, here and below, are fixed for given variables or a
frozen encoder; in our end-to-end LeWM-class models, as in LeWM, gradients also
flow into them. Eq.~\ref{eq:pred} is observational, meaning that no term says which part of the
observed change the action caused.

\paragraph{Why masked prediction is not enough.}
C-JEPA recovers hidden slots from the others, with targets from the same logged
trajectory. This teaches object interactions but does not change the data. If a
nuisance $c$ is correlated with the outcome in training and the state estimate is
imperfect, a predictor can use $c$ and fails when the correlation changes.
\dojepa\ changes the data instead, adding a rollout that differs only in the
action.

\subsection{Paired interventions and what they reveal}\label{sec:pairs}

\paragraph{Paired trajectories.}
From a state $s_t$ we save the full simulator state, including contact caches and
the random-number state, and run it twice for $k$ steps with the same context and
noise. The factual branch applies $a$; the reference branch applies $\anull$ (a
physical no-op unless stated otherwise) and later actions are identical. The result
is a \emph{paired trajectory} that differs only in the intervention, $do(a)$
versus $do(\anull)$; replaying a branch reproduces it exactly (maximum difference
$0.0$). The paired effect at step $t+\ell$ and its prediction are
\begin{equation}
\dz_{t+\ell}=z^{a}_{t+\ell}-z^{\anull}_{t+\ell},
\qquad
\dzh_{t+\ell}=\hat z^{\,a}_{t+\ell}-\hat z^{\,\anull}_{t+\ell},
\qquad \ell=1,\dots,k,
\label{eq:delta}
\end{equation}
both in $\R^{K\times d}$. The target $\dz$ needs both branches and exists only at
training time, whereas the prediction $\dzh$ runs the predictor twice from the same
history. Because the branches share state, context and noise, $\dz$ is the effect
of replacing $\anull$ by $a$ in this state, not an average over states. Any latent
component shared by both branches cancels: exactly for an additive context
contribution, but not for one that interacts with the state. A pair is
\emph{responsive} if the manipulated object moves by more than about $2$\,px
(Push-T) or $2$\,mm (CausalWorld; App.~\ref{app:benchmarks}).

\paragraph{Support, responses and propagation.}
The \emph{direct target} of an intervention is the set of variables it changes at
once. With object-aligned slots, slot $i$ belongs to the \emph{direct support} if
its immediate response $\lVert\dz^{\,i}_{t+1}\rVert$ is large. This label comes
from the pair, not from the simulator's graph, and it is correct only if one
observation step is too short for the effect to reach a second object
(\emph{temporal resolvability}; Sec.~\ref{sec:q12}). The \emph{response set} $R$
contains every slot $j$ with $\lVert\dz^{\,j}_{t+\ell}\rVert>\tau$ for some
$\ell\le k$; its members outside the direct support are the \emph{downstream
responders}, and slots outside $R$ are the \emph{invariant variables} of the
pair. The \emph{onset} of slot $j$ is
$t_j=\min\{\ell:\lVert\dz^{\,j}_{t+\ell}\rVert>\tau\}$. If $j$ starts to respond
one step after $i$, then $i$ is a candidate cause of $j$'s response, which gives a
\emph{propagation label} $y^{\mathrm{edge}}_{ji}\in\{0,1\}$ for the message from
slot $i$ to slot $j$. These labels describe how this intervention propagated. 
Therefore, an interaction that the rollout does not use receives no positive label.

\paragraph{Nuisance and context.}
A \emph{context twin} keeps the state and the action and changes only the
context, $o'=g(s,c')$, so any difference between the predicted effects under $c$
and $c'$ is caused by the context alone. At test time a \emph{nuisance shift}
reverses a training correlation between $c$ and the outcome, a \emph{mechanism
shift} changes physical parameters such as mass, and a
\emph{composed shift} does both; any of these makes a split out-of-distribution
(OOD).

\subsection{Model and objective}\label{sec:model}

\begin{figure}[t]
\centering
\includegraphics[width=0.84\linewidth]{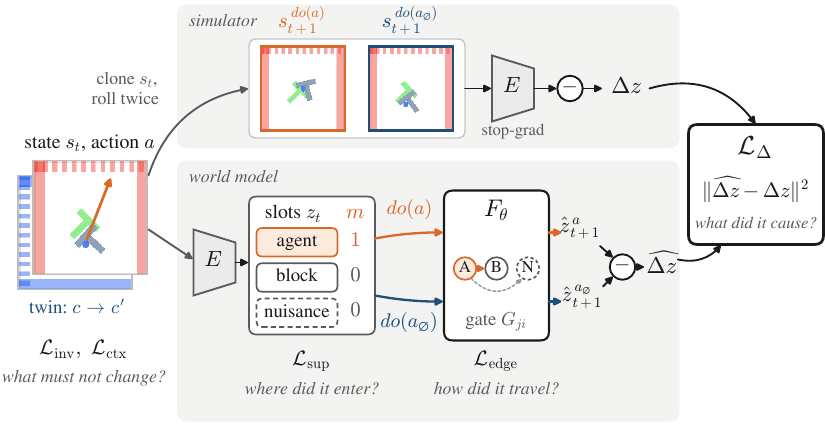}
\caption{\textbf{\dojepa\ training.} Top: the simulator runs the
saved state $s_t$ under $do(a)$ and $do(\anull)$; the encoder $E$ (frozen in the
slot models shown here) gives the target $\Delta z$. Bottom: the world model encodes the observation into
slots $z_t$, routes the action into a slot through the mask $m$ and propagates its
effect through the gates $G_{ji}$; $\dzh$ is the difference of its two predictions.
$\Ld$: what the action caused; $\Lsup$: where it entered; $\Ledge$: how it
travelled; $\Linv,\Lctx$: what must not change.}
\label{fig:overview}
\end{figure}

\paragraph{Architecture.}
The slot predictor for object-aligned variables (Fig.~\ref{fig:overview}) has
three parts. (i) A transformer over the history predicts an action-free next
state $b_{t+1}\in\R^{K\times d}$. (ii) An \emph{action-entry mask}
\begin{equation}
m=\operatorname{softmax}_{i}\big(h_\phi(z^i_t,a)/T_m\big)\cdot\mathbb 1[a\neq\anull]\;\in[0,1]^K ,
\label{eq:mask}
\end{equation}
where $h_\phi$ scores slot $i$ for action $a$ and $T_m$ is a temperature, sums to
one for a real action and is zero for $\anull$. It gates the direct residual
$r^i=m_i\,u_\phi(z^i_t,a)$ of a small network $u_\phi$ with bounded (tanh) output.
(iii) Gated message passing, repeated $P$ times, moves the residual between
slots:
\begin{equation}
r\leftarrow r+\gamma\tanh\!\big(\big((A\odot G)\,r\,W_v^{\top}\big)W_o^{\top}\big),
\qquad G_{ji}=\sigma\big(q_\phi(z^j_t,z^i_t)/T_G\big),\quad G_{jj}=0 ,
\label{eq:prop}
\end{equation}
where $r\in\R^{K\times d}$ has one row per slot, $A$ are attention weights,
$G_{ji}$ gates the message from slot $i$ to slot $j$ and depends on the slot
states only, $W_v,W_o$ are bias-free linear maps and $\gamma$ is a fixed scale.
Gated weights are not renormalised, so a closed gate blocks the message. The
prediction is $\hat z^{\,a}_{t+1}=b_{t+1}+r$. Only $r$ depends on the action and
$m=0$ for $\anull$, so $\dzh_{t+1}=r$. In state-based Push-T the action is known to
move the agent, so it enters the agent slot and $m$ is not learned. The LeWM-class
backbones ($K=1$) have no mask or gates.

\paragraph{Targets and losses.}
All targets come from the paired rollouts or the context twin; none uses the
simulator's graph or object identities (Table~\ref{tab:targets},
App.~\ref{app:notation}). The \emph{effect loss} compares predicted and true
paired effects, $\Ld=\frac1N\sum_{n}\lVert\dzh_n-\dz_n\rVert_2^2$. The support,
propagation, invariance and context losses are
\begin{align}
\Lsup&=-\frac1N\sum_{n}\sum_{i=1}^{K}y^{\mathrm{sup}}_{n,i}\log m_{n,i},
\qquad y^{\mathrm{sup}}_{n,i}=\frac{\lVert\dz^{\,i}_{n,t+1}\rVert^2}{\sum_{i'}\lVert\dz^{\,i'}_{n,t+1}\rVert^2},
\label{eq:lsup}\\
\Ledge&=\frac1N\sum_{n}\Big(\mathrm{BCE}_{\mathrm{bal}}\big(G_n,\,y^{\mathrm{edge}}_n\big)+\lambda_{\mathrm{sp}}\lVert G_n\rVert_1\Big) ,
\label{eq:ledge}\\
\Linv&=\frac1N\sum_{n}\sum_{j\notin R_n}\Big(\big\lVert\dzh^{\,j}_n\big\rVert_2^2+\lambda_{\mathrm{g}}\textstyle\sum_{i}G_{n,ji}\Big),
\label{eq:linv}\\
\Lctx&=\frac1N\sum_{n}\big\lVert\,\dzh_n(c)-\dzh_n(c')\big\rVert_2^2 .
\label{eq:lctx}
\end{align}
$\Lsup$ is a cross-entropy between the mask and the normalised immediate response
energy (an entropy regulariser keeps $m$ sharp); $\Ledge$ is a class-balanced binary
cross-entropy between gates and onset labels with an $\ell_1$ penalty; $\Linv$
pushes predicted effects on non-responders to zero and closes the gates into them;
$\Lctx$ asks the predicted effect to be the same under a context and its twin (in
state-based Push-T it compares factual predictions of the physical slots; in the
end-to-end models it averages Eq.~\ref{eq:lctx} with the same term on the target
effects, without stopping gradients).

\paragraph{Full objective and model variants.}
The full objective is
\begin{equation}
\mathcal L=\Lpred(a)+\lambda_{\mathrm{ref}}\,\Lpred(\anull)+\lambda_\Delta\Ld
+\lambda_{\mathrm{sup}}\Lsup+\lambda_{\mathrm{edge}}\Ledge
+\lambda_{\mathrm{inv}}\Linv+\lambda_{\mathrm{ctx}}\Lctx+\mathcal R ,
\label{eq:full}
\end{equation}
where $\Lpred(a)$ and $\Lpred(\anull)$ apply Eq.~\ref{eq:pred} to the two
branches and $\mathcal R$ collects the backbone's regulariser and the
setting-specific terms of App.~\ref{app:training}. The variants differ only
in how they use the same collected data. \Obs\ uses the factual branch only. \Aug,
the \emph{information-matched control}, uses the factual branch, the reference
branch and, where used, the context twin as ordinary examples
($\lambda_{\mathrm{ref}}>0$, $\lambda_\Delta=0$) but never compares them. \Paired\
adds $\Ld$; \dojepa\ (full) adds $\Lctx$ and, in slot models, the slot losses
listed for each setting in Table~\ref{tab:benchmarks}. \Mask\ is our re-implementation of C-JEPA on the
same slots.

\paragraph{Training, inference and information.}
Training uses paired trajectories (plus context twins for $\Lctx$). At test time
the model receives only an observation history and an action; it predicts an
effect by running the predictor with $a$ and with $\anull$ from the same history
(Eq.~\ref{eq:delta}), and planning uses only the factual prediction. The losses
need an exactly restorable simulator (all
losses), object-aligned slots ($\Lsup$, $\Ledge$, $\Linv$) and a way to change the
context without changing the physics ($\Lctx$). \Paired\ and \Aug\ receive the same
branches, so their comparison is matched in information. In the SCMs and
state-based Push-T the model's variables are the physical state; in the pixel
settings physical state is used only to build benchmarks and to evaluate
(Table~\ref{tab:info}).

\section{Experiments}\label{sec:experiments}

We try to answer four questions, each in the setting that isolates it: \textbf{Q1:}
where does the action enter? \textbf{Q2:} where does its effect travel? \textbf{Q3:} what must not change? \textbf{Q4:} does the
effect loss help when the representation is learned from pixels?
App.~\ref{app:benchmarks}
and~\ref{app:training} describe the settings and training. \emph{Protocol.} Checkpoints and hyperparameters are
selected on in-distribution (IID) validation data, except in two development runs
that we mark where they occur. In \emph{preregistered}
experiments, hypotheses, thresholds and seeds were fixed in a protocol file before
data collection or before the OOD splits were opened, and OOD splits were
evaluated once after the checkpoint was locked. \emph{Confirmed} means that a
result holds on at least $2$ of $3$ fresh seeds. App.~\ref{app:metrics} defines the metrics.


\subsection{Q1--Q2: Where does the action enter, and where does its effect travel?}\label{sec:q12}

\paragraph{Question and setup.}
Can a model find the directly intervened object, and the paths along which its
effect spreads, from paired data alone? Sparsity is often assumed to localise an
intervention; if it did, support labels would be unnecessary. A synthetic
structural causal model (SCM) gives known ground truth: four objects on a directed
ring ($i\to i{+}1 \bmod 4$) and three nuisance slots. The action, an impulse at a
noisy contact point biased towards the next object, does not name its target; the
first step changes only the direct target, and the effect then spreads along the
ring for two steps. A hidden context makes the nuisances co-vary with the action
in training ($\rho=0.95$); the test split reverses this. We compare \Mask\ (global
action), \Mask$+\Ld$, a sparse action mask, the mask $+\Ld$ and the mask
$+\Ld+\Lsup$\pid{hard}, then gates without and with $\Ledge$ and the gate term of
$\Linv$\pid{hard-edge}. In state-based Push-T, where the agent is the known direct
target, we ask whether the gate $G_{\text{block}\leftarrow\text{agent}}$ learns
\emph{when} the agent moves the block\pid{P1}.

\begin{table}[t]
\centering
\caption{\textbf{Where does the action enter?} Synthetic SCM,
reversed-correlation test split, seed 7. Top-1 over all seven slots (chance
$1/7$) and, in brackets, over the four objects (chance $0.25$); nuis.\ mask: mask
mass on nuisance slots; nuis.\ effect: mean RMS of the predicted effect per
nuisance slot. All rows are evaluated with one history slot masked, as in \Mask\
training; mask temperature is $1.0$ in the sparse-mask rows and $0.7$ in the last
(App.~\ref{app:synthetic}). Best in bold, second best underlined.}
\label{tab:synth}
\setlength{\tabcolsep}{3.5pt}
\resizebox{\linewidth}{!}{%
\begin{tabular}{@{}lcccccc@{}}
\toprule
& Top-1 all [obj.]\,$\uparrow$ & Target $F_1$\,$\uparrow$ & Nuis.\ mask\,$\downarrow$ & Nuis.\ effect\,$\downarrow$ & Effect MSE\,$\downarrow$ & Pred.\ MSE\,$\downarrow$ \\
\midrule
\Mask\ (global action)            & --- & --- & --- & $\underline{6.3\times10^{-3}}$ & $1.5\times10^{-3}$ & $\underline{3.14\times10^{-3}}$ \\
\Mask\ $+\,\Ld$                    & --- & --- & --- & $\mathbf{2.9\times10^{-3}}$ & $\mathbf{3.5\times10^{-4}}$ & $\mathbf{2.67\times10^{-3}}$ \\
Sparse mask                       & $0.000$ [$0.243$] & $0.000$ & $1.000$ & $2.5\times10^{-2}$ & $1.1\times10^{-2}$ & $8.70\times10^{-3}$ \\
Sparse mask $+\,\Ld$              & $0.000$ [$0.234$] & $0.000$ & $1.000$ & $1.6\times10^{-2}$ & $1.0\times10^{-2}$ & $8.55\times10^{-3}$ \\
Sparse mask $+\,\Ld+\Lsup$        & $\mathbf{0.9995}$ [$0.9995$] & $\mathbf{0.9998}$ & $\mathbf{8.2\times10^{-5}}$ & $9.1\times10^{-3}$ & $\underline{9.3\times10^{-4}}$ & $3.43\times10^{-3}$ \\
\bottomrule
\end{tabular}}
\end{table}

\paragraph{Result: the direct target.}
A sparse mask selects exactly one slot, as the sparsity term asks, but in every
test case it selects a nuisance (top-1 $0.000$ over all slots; $0.243$ if only
objects are ranked; Table~\ref{tab:synth})\pid{hard}. Adding $\Ld$ does not help:
the propagation step can move the effect wherever it is needed, so any entry slot
satisfies the effect loss. Adding $\Lsup$, together with a lower mask temperature
(we did not run a temperature-matched control), gives $0.9995$ with almost no mask
mass on nuisances. Sparsity sets how many slots the action enters; here the
immediate paired response set which one. Prediction and effect error do not reveal this
(\Mask$+\Ld$ has the lowest of both), and the support-trained model still predicts
effects on nuisances ($9.1\times10^{-3}$) because nothing constrains how its
effect spreads. An easier SCM shows the same pattern over five seeds
(App.~\ref{app:synthetic})\pid{W1-2}.

\paragraph{Result: the propagation graph.}
Gates trained without $\Ledge$ are open almost everywhere (edge AUROC $0.624$
against the test pairs' onset labels; Table~\ref{tab:synth-edge})\pid{hard-edge}.
With $\Ledge$ it is $0.975$ on all three seeds, and thresholding the mean test gate at $0.5$ recovers exactly
the four ring edges on each seed (mean gate $0.95$ on ring edges, $0.002$--$0.005$ on
other pairs), and predicted nuisance effects fall to $1.9\times10^{-5}$, at the
cost of a less accurate direct effect (MSE $0.0137$ vs.\ $0.0033$). In
state-based Push-T the gate is state-dependent: edge AUROC is $0.9944\pm0.0002$
against $0.500$ without gates, and the predicted action effect on nuisances falls
from $1.9\times10^{-3}$ (\Obs) to $2.1\times10^{-7}$ while prediction MSE differs
by at most $2\%$ across models (Table~\ref{tab:pusht-state})\pid{P1}.

\paragraph{Result: time resolution.}\label{sec:temporal}
Onset labels assume that causal events are separated in time. In a second SCM (six
objects on a ring), $\kappa\in\{1,2,3,4\}$ of four propagation stages fall inside
one observed step (App.~\ref{app:temporal})\pid{P13}. As $\kappa$ grows, one-hop
edge $F_1$ collapses ($0.999\to0.549\to0.502\to0.337$), while direct-target $F_1$
stays at $0.997$ and prediction MSE at $0.0103$--$0.0105$. The onset labels
themselves, scored as one-hop edges, fall from $F_1=1.00$ to $0.25$: with two
transitions between observations, a chain $A\to B\to C$ and a fork
$\{A\to B, A\to C\}$ produce the same data. Scored against reachability within
the observation interval, the same gates reach $F_1=0.999/0.935/0.992/0.894$\pid{P13.1}, and one gate head per time scale, tied by a reachability constraint,
represents both levels (one-hop $F_1=0.977$; coarse $0.909$--$0.977$)\pid{P13.2}.

\paragraph{Takeaway.}
In our synthetic systems, with object-aligned variables and resolvable onsets, the
immediate paired response identified the direct target and onset order recovered
the ring graph; sparsity and the effect loss did not, and with coarser observation
the recoverable relation became reachability, which prediction error does not
reveal. On anonymous frozen slots of pixel Push-T, which have no ground-truth
graph, a single-seed IID diagnostic gives onset-trained gates that reproduce the
test onset labels (AUROC $1.000$ vs.\ $0.823$ without $\Ledge$;
App.~\ref{app:negatives})\pid{P23.1, P23.2}.

\subsection{Q3: What must not change?}\label{sec:q3}

The action-side losses do not stop a nuisance from changing the prediction. We
test whether context twins remove a shortcut through a nuisance that is
correlated with the outcome in training.

\paragraph{With given variables.}
In state-based Push-T a nuisance vector is correlated with contact geometry
($\rho=0.95$) in training and reversed at test, and the physical inputs are
corrupted with sensor noise\pid{P2}. \Obs\ degrades by $2.9\times$ in prediction
MSE and the gated model by $3.1\times$, although the latter's effect error changes
only by $1.02\times$. We chose $\lambda_{\mathrm{ctx}}=0.25$ on seed 7's IID--OOD
trade-off curve, so that seed's OOD split informed the choice\pid{P4}. On the two other seeds, $\Lctx$ lowers OOD error by
$43.9\%$ and $55.2\%$ and closes $71.9\%$ and $85.4\%$ of the IID-to-OOD gap, at
$4.2\%$ and $3.5\%$ higher IID error, with edge AUROC unchanged
(App.~\ref{app:pusht-results}).

\paragraph{Boundary condition: learned pixel latents.}
On learned latents the same idea needs care. Applied as a generic, isotropic
penalty on latent effects, $\Lctx$ lowers the latent context shift of predicted
effects ($20\%$ on frozen VideoSAUR slots, $46$--$47\%$ end to end in CausalWorld)
but costs factual and effect accuracy, and on frozen slots the shift decoded to
pixels does not fall\pid{P12.5, P20.12}. Diagnostics trace this to a mismatch
between latent and physical distance: two latent directions carry about $99.8\%$
of the decoded physical error, and a predictor-side variant of the penalty leaves
$10$--$19\%$ more residual energy exactly there\pid{P20.17}
(App.~\ref{app:cw-results}). Pairing alone already lowers the latent context shift
by about $66\%$ without any context term (Sec.~\ref{sec:q4}).

\paragraph{Takeaway.}
With given variables, nuisance interventions remove most of the shortcut at a
small IID cost. On learned latents, invariance has to be measured in the
directions that carry the physics; a generic latent penalty is not enough.

\subsection{Q4: Does the effect loss help from pixels?}\label{sec:q4}

On C-JEPA's own frozen VideoSAUR slots of pixel Push-T, \dojepa\ beats C-JEPA-style
masking on all six metrics and all three seeds: factual error is $30\%$ lower,
responsive effect error $12\%$ lower, effect cosine $0.14$ higher, response AUROC
$0.21$ higher (masking is at chance) and the context shift of predicted effects
$31\%$ lower in pixels (App.~\ref{app:vision})\pid{P12.5}. We test the effect loss
itself where the representation is learned, end to end, against \Aug, which sees
exactly the same branches; this control matters because \Paired\ receives both
branches.

\paragraph{LeWM on pixel Push-T.}
We fine-tune the released LeWM checkpoint \citep{maes2026lewm} for $10$ epochs on
$10{,}000$ exact same-state pairs (a $5$-step action macro vs.\ a zero macro; half
responsive), with and without $\Ld$ (weight $0.5$). The preregistered rule needs
$2$ of $3$ fresh seeds to lower effect RMSE by $5\%$ or raise effect cosine by
$0.05$\pid{P22}. All three pass (Fig.~\ref{fig:paired}): effect RMSE falls by
$28.4\%$ ($0.213\to0.153$; per seed $27.5$--$29.0\%$), cosine over all pairs rises
by $0.066$, and factual MSE also falls, by $39.7\%$. All are latent metrics, each
in its own model's latent space. On a fresh holdout of $5{,}000$ pairs
with actions drawn independently from the expert action distribution and no
outcome filtering ($6.5\%$ responsive), effect RMSE is $11.2\%$ lower ($3/3$
seeds) and factual MSE $17.5\%$ lower, mostly on non-responsive
pairs\pid{P22B}. A physical read-out, a linear map from latent effects to block
displacement fitted per model, confirms the gain on the training distribution
(responsive physical effect error $7\%$ lower on every seed) but not on the
holdout's responsive pairs ($3\%$ higher; App.~\ref{app:lewm-seeds})\pid{P25.0}. The gain follows the training
intervention distribution. Trained with the same recipe on natural expert action
sequences, \Paired\ lowers the responsive physical effect error by $13.5\%$ on an
episode-disjoint natural test set and by $5.5\%$ on the outcome-unfiltered
holdout, on every seed\pid{P25.3}.

\begin{figure}[t]
\centering
\includegraphics[width=\linewidth]{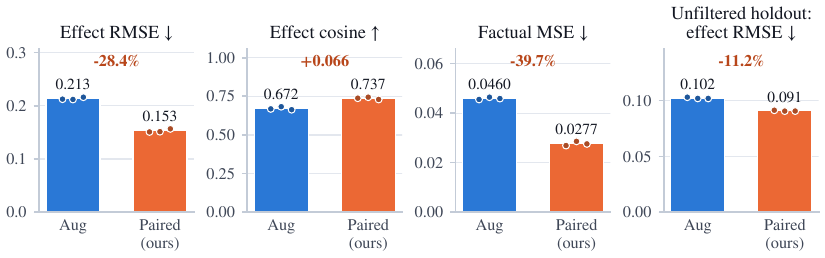}
\caption{\textbf{Effect loss on LeWM, pixel Push-T.} \Aug\ and \Paired\ see
the same paired images; \Paired\ adds $\Ld$. Latent metrics; means over three
fresh seeds (dots: seeds); red: relative change (absolute for the cosine, taken
over all pairs). Panels 1--3: held-out pairs from the training distribution (half
responsive); panel 4: outcome-unfiltered holdout ($6.5\%$ responsive).}
\label{fig:paired}
\end{figure}

\paragraph{CausalWorld: three independently generated benchmarks.}
In the CausalWorld pushing task \citep{ahmed2021causalworld} a robot finger pushes
a block; the branches differ only in the first 9-D joint command, which the
reference mirrors about the current joint position. OOD splits reverse a visual
shortcut (\emph{visual}), change mass and friction while making the context
independent of the outcome (\emph{mechanism}), or change the physics and reverse
the shortcut (\emph{composed}). A first version of the benchmark
leaked the label through how samples were drawn (sampling-only AUROC $0.834$). We
rebuilt it with outcome-defined labels and exact covariate balance (AUROC $0.50$;
App.~\ref{app:causalworld})\pid{P20.3--P20.5}, regenerated it from three dataset
seeds, and fixed in advance the expected signature of \Paired\ against \Aug:
lower responsive effect error and context shift and higher cosine (benefit),
higher factual error and lower response AUROC (cost)\pid{P21}. Checkpoints were
locked on IID data before the OOD splits were opened once\pid{P21B}.

\begin{table}[t]
\centering\small
\caption{\textbf{CausalWorld: \Paired\ vs.\ \Aug\ on three independently
generated benchmarks} (dataset seeds 17/27/37, model seed 107, $40$k updates),
locked OOD splits. Mean relative changes; brackets count instances that move in
the stated direction; ``Signature'' counts instances with the full preregistered
benefit and cost. $^{*}$The mechanism split also decorrelates the visual context
(App.~\ref{app:causalworld}).}
\label{tab:cw}
\setlength{\tabcolsep}{3.5pt}
\resizebox{\linewidth}{!}{%
\begin{tabular}{@{}lcccccc@{}}
\toprule
& \multicolumn{3}{c}{\textbf{Benefit}} & \multicolumn{2}{c}{\textbf{Cost}} & \\
\cmidrule(lr){2-4}\cmidrule(lr){5-6}
\textbf{OOD split} & Resp.\ effect RMSE & Effect cosine & Latent ctx.\ shift & Factual RMSE & Response AUROC & \textbf{Signature} \\
\midrule
Visual    & $3.9\%$ lower [2/3]  & $+0.018$ [3/3] & $66.2\%$ lower [3/3] & $52.8\%$ higher [3/3] & $-0.054$ [3/3] & $2/3$ \\
Mechanism$^{*}$ & $20.3\%$ lower [3/3] & $+0.030$ [2/3] & $65.4\%$ lower [3/3] & $51.5\%$ higher [3/3] & $-0.070$ [3/3] & $2/3$ \\
Composed  & $19.5\%$ lower [3/3] & $+0.020$ [2/3] & $66.2\%$ lower [3/3] & $51.1\%$ higher [3/3] & $-0.073$ [3/3] & $2/3$ \\
\bottomrule
\end{tabular}}
\end{table}

\paragraph{Result.}
The preregistered signature holds on all three OOD splits
(Table~\ref{tab:cw})\pid{P21B}. \Paired\ lowers the context shift of predicted
effects by about $66\%$ without any context loss, and lowers responsive effect
error by about $20\%$ under the mechanism and composed shifts, at the cost of
$51$--$53\%$ higher factual error and lower response AUROC. Three fresh model
seeds on one benchmark instance show the same benefit and cost. In these runs the factual
cost ($15$--$18\%$) exceeds the preregistered $10\%$ tolerance. 


\paragraph{Where the factual cost comes from.}
In Push-T the effect loss fine-tunes a pretrained model and in CausalWorld it shapes
the representation from scratch. To separate the two, we fine-tune each converged
\Aug\ model for $2{,}000$ updates with and without $\Ld$ (the Push-T regime) and
compare the last checkpoints\pid{P25.1}. The factual cost disappears: factual
error changes by $-4$ to $+2\%$ and response AUROC is unchanged on every instance,
while responsive effect error falls by $8\%$ IID and $13\%$ under the physics
shifts and the effect cosine rises on every instance and split
(Table~\ref{tab:cw-regimes}). From scratch, the effect weight sets the trade-off:
$\alpha=0.25$ cuts the factual cost from $52\%$ to $10\%$ and the AUROC loss from
$0.066$ to $0.016$, and keeps the IID effect gain and about two thirds of the
gain under physics shifts\pid{P25.2}.

\begin{table}[t]
\centering\small
\caption{\textbf{When does the effect loss cost factual accuracy?} CausalWorld,
\Paired\ against the \Aug\ model of the same regime, mean relative change over the
three benchmark instances. OOD: mean over the three OOD splits; physics: mechanism
and composed splits. Last column: instances whose factual error rises by at most
$10\%$ on every split. The last two rows are post-hoc diagnostics
(App.~\ref{app:cw-results}).}
\label{tab:cw-regimes}
\setlength{\tabcolsep}{3pt}
\resizebox{\linewidth}{!}{%
\begin{tabular}{@{}lccccccc@{}}
\toprule
& \multicolumn{2}{c}{Factual RMSE} & \multicolumn{2}{c}{Resp.\ effect RMSE} & Response & Latent ctx. & Factual \\
\cmidrule(lr){2-3}\cmidrule(lr){4-5}
Regime & IID & OOD & IID & Physics & AUROC & shift & within $10\%$ \\
\midrule
From scratch, $\alpha=1$ & $+51.9\%$ & $+51.8\%$ & $-9.0\%$ & $-19.9\%$ & $-0.066$ & $-66.0\%$ & $0/3$ \\
From scratch, $\alpha=0.25$ & $+10.3\%$ & $+10.6\%$ & $-9.1\%$ & $-12.8\%$ & $-0.016$ & $-33.6\%$ & $1/3$ \\
Fine-tuned from \Aug, $\alpha=1$ & $-0.6\%$ & $-0.9\%$ & $-8.2\%$ & $-13.3\%$ & $+0.008$ & $-26.2\%$ & $3/3$ \\
\bottomrule
\end{tabular}}
\end{table}

\paragraph{Takeaway.}
In every comparison on end-to-end LeWM-class models, the effect loss lowered mean
effect error relative to an information-matched control\pid{P18.1, P20.7, P20.12, P21B, P22}.
Its effect on factual accuracy depends on the training regime: fine-tuning an
existing model with it keeps factual accuracy (CausalWorld) or improves it
(Push-T), while training from scratch trades factual accuracy and response AUROC
for a larger effect gain, in proportion to the effect weight. 

\section{Discussion and Limitations}\label{sec:discussion}

We show direct-target detection and responder prediction with object-aligned
variables (Q1--Q2), nuisance robustness with given variables (Q3), and better
effect prediction from pixels than an information-matched control, including
lower responsive effect error under CausalWorld physics shifts (Q4). Paired
supervision can also be added to a pretrained LeWM planner through a centered
adapter, with success equivalent to the base planner within $\pm3$ percentage
points (App.~\ref{app:planning}). 
Our supervision signals capture how interventions propagate in observed rollouts, under the training intervention distribution and at the observation rate. Learning them does not require identification of the underlying mechanism, and we do not claim it. Here $do(\cdot)$ names interventions and we use no do-calculus. Paired effects are counterfactual in the sense that both branches share state and noise in a simulator.

\paragraph{Limitations.}
\dojepa\ rests on several assumptions. First, exact pairs require a simulator whose
full state can be saved and restored, and each pair costs two rollouts (a context
twin changes only the observation, not the rollout). Real trials share neither
noise nor exact state, so differences between them are noisy and may be biased.
Second, the intervention and its reference must be known, although the direct
target need not be. Third, support and onset labels require object-aligned slots,
a known slot count, a single direct target per action (the support head is a
softmax) and a frame rate at which response onsets separate in time. Finally,
context twins assume that the context can change without changing the physics,
and the nuisances in our benchmarks are rendered or synthetic. Relaxing these
assumptions, for example with approximate real-world pairs or anonymous object
slots, is a promising direction for future work.


\section{Conclusion}\label{sec:conclusion}
In this work, we proposed \dojepa, a framework for training latent world models
with interventions on the world rather than on what the model sees. From one saved
simulator state, \dojepa\ contrasts an action with a reference action and learns
what the action changed, where it entered, where its effect travelled and what
must stay invariant. With object-aligned variables in synthetic systems, the
immediate paired response identifies where an action enters, and onset order
identifies where its effect travels, down to the time resolution at which events
separate. From pixels, the paired effect loss improves effect prediction over a
control trained on the same data in two environments, and fine-tuning with it
preserved factual accuracy in our experiments. Context twins provide nuisance
invariance with given variables, whereas on learned latents invariance must be
measured in the directions that carry the physics. Finally, prediction error
revealed little about these properties, which suggests that interventional metrics
should become a routine part of world-model evaluation.

\newpage




\IfFileExists{iclr2027_conference.bst}{\bibliographystyle{iclr2027_conference}}%
                                      {\bibliographystyle{plainnat}}
\bibliography{refs}

\appendix

\newcommand{\field}[1]{\textit{#1:}~}

\clearpage
\section*{Contents of the Appendix}
\vspace{-2pt}
{\small
\vspace{3pt}\noindent\hyperref[app:notation]{\textbf{\ref*{app:notation}\quad Notation and Terminology}}\dotfill\textbf{\pageref*{app:notation}}\par
\vspace{3pt}\noindent\hyperref[app:metrics]{\textbf{\ref*{app:metrics}\quad Metrics Guide}}\dotfill\textbf{\pageref*{app:metrics}}\par
\vspace{3pt}\noindent\hyperref[app:benchmarks]{\textbf{\ref*{app:benchmarks}\quad Benchmarks}}\dotfill\textbf{\pageref*{app:benchmarks}}\par
\noindent\hspace*{1.6em}\hyperref[app:scm]{\ref*{app:scm}\quad Synthetic SCMs}\dotfill\pageref*{app:scm}\par
\noindent\hspace*{1.6em}\hyperref[app:pusht-state]{\ref*{app:pusht-state}\quad State-based Push-T}\dotfill\pageref*{app:pusht-state}\par
\noindent\hspace*{1.6em}\hyperref[app:vision-bench]{\ref*{app:vision-bench}\quad Vision Push-T}\dotfill\pageref*{app:vision-bench}\par
\noindent\hspace*{1.6em}\hyperref[app:corpora]{\ref*{app:corpora}\quad LeWM Push-T pairs and planning corpora}\dotfill\pageref*{app:corpora}\par
\noindent\hspace*{1.6em}\hyperref[app:causalworld]{\ref*{app:causalworld}\quad CausalWorld pushing}\dotfill\pageref*{app:causalworld}\par
\vspace{3pt}\noindent\hyperref[app:training]{\textbf{\ref*{app:training}\quad Training and Evaluation Details}}\dotfill\textbf{\pageref*{app:training}}\par
\vspace{3pt}\noindent\hyperref[app:results]{\textbf{\ref*{app:results}\quad Additional Results}}\dotfill\textbf{\pageref*{app:results}}\par
\noindent\hspace*{1.6em}\hyperref[app:synthetic]{\ref*{app:synthetic}\quad Synthetic SCMs}\dotfill\pageref*{app:synthetic}\par
\noindent\hspace*{1.6em}\hyperref[app:temporal]{\ref*{app:temporal}\quad Temporal resolvability}\dotfill\pageref*{app:temporal}\par
\noindent\hspace*{1.6em}\hyperref[app:pusht-results]{\ref*{app:pusht-results}\quad State-based Push-T}\dotfill\pageref*{app:pusht-results}\par
\noindent\hspace*{1.6em}\hyperref[app:vision]{\ref*{app:vision}\quad Vision Push-T}\dotfill\pageref*{app:vision}\par
\noindent\hspace*{1.6em}\hyperref[app:lewm-seeds]{\ref*{app:lewm-seeds}\quad LeWM Push-T: per-seed results}\dotfill\pageref*{app:lewm-seeds}\par
\noindent\hspace*{1.6em}\hyperref[app:cw-results]{\ref*{app:cw-results}\quad CausalWorld}\dotfill\pageref*{app:cw-results}\par
\vspace{3pt}\noindent\hyperref[app:planning]{\textbf{\ref*{app:planning}\quad Planning Study Details}}\dotfill\textbf{\pageref*{app:planning}}\par
\vspace{3pt}\noindent\hyperref[app:negatives]{\textbf{\ref*{app:negatives}\quad Boundary Conditions and Diagnostic Findings}}\dotfill\textbf{\pageref*{app:negatives}}\par
}
\vspace{6pt}

\begin{table}[h]
\centering\small
\caption{\textbf{Main results at a glance.} Each row gives the strongest evidence for
one finding and where it is reported. Seeds are model seeds; ``instances'' are
independently generated benchmarks.}
\label{tab:glance}
\setlength{\tabcolsep}{4pt}
\renewcommand{\arraystretch}{1.12}
\begin{tabularx}{\linewidth}{@{}>{\raggedright\arraybackslash}p{0.27\linewidth}>{\raggedright\arraybackslash}X>{\raggedright\arraybackslash}p{0.13\linewidth}@{}}
\toprule
\textbf{Finding} & \textbf{Evidence} & \textbf{Where} \\
\midrule
Support supervision finds where the action enters & direct target found in $99.95\%$ of test cases; a sparse mask picks a nuisance slot in every case & Table~\ref{tab:synth} \\
Onset supervision recovers where the effect travels & edge AUROC $0.975$ vs.\ $0.624$; the four ring edges recovered exactly on $3/3$ seeds & Table~\ref{tab:synth-edge} \\
Gates confine effects to the right objects & state Push-T edge AUROC $0.994$ vs.\ $0.500$; predicted effect on nuisances $1.9\times10^{-3}\to2.1\times10^{-7}$ & Table~\ref{tab:pusht-state} \\
Context twins remove a nuisance shortcut & on the two seeds not used to pick the weight: OOD error $44$--$55\%$ lower, $72$--$85\%$ of the IID-to-OOD gap closed, IID cost $\le4.2\%$ & App.~\ref{app:pusht-results} \\
Intervention beats masking on C-JEPA's own slots & all six metrics, $3/3$ seeds: factual error $30\%$ lower, AUROC $+0.21$, pixel context shift $31\%$ lower & Table~\ref{tab:vision} \\
Effect loss helps from pixels (LeWM) & latent effect error $28.4\%$ lower ($3/3$); physical effect error $13.5\%$ lower when trained and tested on natural sequences ($3/3$) & Fig.~\ref{fig:paired}, Table~\ref{tab:pusht-dist} \\
Effect loss helps under physics shifts (CausalWorld) & responsive effect error ${\approx}20\%$ lower and latent context shift $66\%$ lower, on three benchmark instances & Table~\ref{tab:cw} \\
Fine-tuning keeps factual accuracy & factual error $-4$ to $+2\%$, effect error $8$--$13\%$ lower, on $3/3$ instances & Tables~\ref{tab:cw-regimes},~\ref{tab:cw-ft} \\
Compatible with a pretrained planner & success within $\pm3$ points of the base planner ($600$ episodes) & App.~\ref{app:planning} \\
\bottomrule
\end{tabularx}
\end{table}

\section{Notation and Terminology}\label{app:notation}

Table~\ref{tab:notation} lists every symbol used in the paper. The column
``Role'' says whether a quantity is \emph{observed} (produced by the simulator or
the encoder), \emph{predicted} (produced by the predictor), a \emph{target}
(used inside a loss), a \emph{label} (a target derived from paired rollouts), or
a \emph{parameter}. The column ``When'' says whether it exists at training time,
at test time, or only for evaluation.

\begin{table}[p]
\centering\small
\caption{\textbf{Notation.} $K$: number of slots; $d$: slot dimension; $H$:
history length; $k$: rollout horizon; $N$: batch size.}
\label{tab:notation}
\setlength{\tabcolsep}{3pt}
\footnotesize\renewcommand{\arraystretch}{0.96}
\begin{tabularx}{\linewidth}{@{}l>{\raggedright\arraybackslash}Xlll@{}}
\toprule
\textbf{Symbol} & \textbf{Meaning} & \textbf{Shape} & \textbf{Role} & \textbf{When} \\
\midrule
$s_t$ & physical simulator state & env.-specific & observed & build/eval \\
$a,\ \anull$ & action; reference action (a physical no-op unless stated) & $\mathcal A$ & input & train, test \\
$c,\ c'$ & context (nuisance) value and its twin & env.-specific & input & train (twin), test ($c$) \\
$u$ & confounder: affects action choice and outcome & scalar & hidden & build \\
$\varepsilon_t$ & exogenous simulator noise, shared by both branches & --- & hidden & build \\
$o_t$ & observation $g(s_t,c)$ (image or state vector) & env.-specific & observed & train, test \\
$z_t=(z^1_t,\dots,z^K_t)$ & encoder output, one row per slot & $K\times d$ & observed & train, test \\
$z^{a}_{t+\ell},\ z^{\anull}_{t+\ell}$ & encodings of the factual and reference branches & $K\times d$ & target & train \\
$\hat z^{\,a}_{t+\ell},\ \hat z^{\,\anull}_{t+\ell}$ & predictions for the two branches & $K\times d$ & predicted & train, test \\
$\dz$ & true paired effect, Eq.~\ref{eq:delta} & $K\times d$ & target & train \\
$\dzh$ & predicted paired effect, Eq.~\ref{eq:delta} & $K\times d$ & predicted & train, test \\
$\Delta p,\ \Delta\hat p$ & true and decoded physical effect of the manipulated object & $2$ or $3$ & observed / predicted & eval \\
$m$ & action-entry mask, Eq.~\ref{eq:mask} & $K$ & predicted & train, test \\
$r$ & direct (then propagated) residual & $K\times d$ & predicted & train, test \\
$G_{ji}$ & gate on the message from slot $i$ to slot $j$, Eq.~\ref{eq:prop} & $K\times K$ & predicted & train, test \\
$A$ & attention weights between slots & $K\times K$ & predicted & train, test \\
$y^{\mathrm{sup}}$ & normalised immediate response energy (support label) & $K$ & label & train \\
$y^{\mathrm{edge}}_{ji}$ & onset-order propagation label & $K\times K$ & label & train \\
$R$ & response set: slots that respond within $k$ steps & set & label & train \\
$T$ & true direct target & set & observed & eval only \\
$E_\psi,\ F_\theta$ & encoder and predictor & --- & parameter & --- \\
$h_\phi,\ u_\phi,\ q_\phi$ & mask scorer, direct-effect network, gate scorer & --- & parameter & --- \\
$\Fb,\ \Fa,\ g_\phi$ & frozen planner predictor, adapted predictor, adapter & --- & parameter & --- \\
$\kappa$ & number of causal stages hidden in one observed step & integer & design & build \\
$\lambda_{(\cdot)},\ \beta$ & loss weights ($\beta$ = context weight in CausalWorld) & scalar & hyperparameter & train \\
\bottomrule
\end{tabularx}
\end{table}

\paragraph{Terminology.} We use each term in one sense only.
\begin{description}\itemsep1pt
\item[Intervention, $do(a)$.] Setting the action to a chosen value in a
      simulator from a saved state. We use the notation to name interventions,
      not to invoke do-calculus.
\item[Reference action $\anull$.] The action of the second branch. It is a
      physical no-op in Push-T and the synthetic SCMs, a mirrored joint command
      in CausalWorld, and a different real expert action in the reference-contrast
      test (App.~\ref{app:planning}).
\item[Paired trajectory.] Two rollouts from one saved state with shared context
      and noise that differ only in the intervention. \emph{Counterfactual} refers
      only to this construction.
\item[Direct intervention, intervention target, direct target.] The variables an
      intervention changes at once. \emph{Intervention support} (direct
      support) is its estimate from the first step of a pair.
\item[Response, responsive component.] A variable \emph{responds} if its paired
      effect exceeds a threshold within the horizon; the \emph{response set}
      contains all such variables. A \emph{responsive pair} is a pair in which the
      manipulated object moves by more than a threshold: $2$\,px in vision
      Push-T, $2$\,px or $0.02$\,rad in the LeWM Push-T corpora, and $2$\,mm in
      CausalWorld, where pairs between $0.5$ and $2$\,mm are discarded
      (App.~\ref{app:benchmarks}).
\item[Downstream (propagated) effect.] A response of a variable that is not a
      direct target.
\item[Temporal support, onset.] The step at which a variable starts to respond.
      \emph{Causal propagation} is the set of ordered pairs $(i,j)$ such that $j$
      starts to respond one step after $i$.
\item[Invariant variable, invariance.] A variable outside the response set of a
      pair; the model should predict zero effect on it.
\item[Context, nuisance context.] A variable that changes the observation but not
      the physics. \emph{Context invariance} means that predicted effects do not
      change when only the context changes.
\item[Mechanism, mechanism shift.] The physical parameters of the dynamics
      (mass, friction, moment of inertia); a mechanism shift changes them at test
      time. A \emph{composed shift} changes both mechanism and context.
\item[OOD.] A test split with a nuisance, mechanism, composed or confounder
      shift; IID means the training distribution.
\end{description}

\paragraph{Targets and information flow.} Table~\ref{tab:targets} lists each loss
with its target and the source of that target, and Table~\ref{tab:info} lists where
each kind of information enters training, testing and evaluation.

\begin{table}[h]
\centering\small
\caption{\textbf{Losses and their targets.} Every target is computed from paired
rollouts or context twins at training time. $e_i=\lVert\dz^{\,i}_{t+1}\rVert^2$
is the immediate response energy of slot $i$; $R$ is the response set;
``onset order'' is defined in Sec.~\ref{sec:pairs}.}
\label{tab:targets}
\setlength{\tabcolsep}{4pt}
\begin{tabular}{@{}llll@{}}
\toprule
\textbf{Loss} & \textbf{Question} & \textbf{Target} & \textbf{Source of the target} \\
\midrule
$\Lpred$ & what happens next? & $z^{a}_{t+1}$, $z^{\anull}_{t+1}$ & observed next frames \\
$\Ld$ & what did the action change? & $\dz=z^{a}-z^{\anull}$ & paired rollouts, Eq.~\ref{eq:delta} \\
$\Lsup$ & where did it enter? & $y^{\mathrm{sup}}_i=e_i/\sum_{i'}e_{i'}$ & first step of the pair \\
$\Ledge$ & where did it travel? & $y^{\mathrm{edge}}_{ji}\in\{0,1\}$ & onset order in the pair \\
$\Linv$ & what must not change? & $\dzh^{\,j}=0$ for $j\notin R$ & response set of the pair \\
$\Lctx$ & what must ignore context? & $\dzh(c)=\dzh(c')$ & context twin \\
\bottomrule
\end{tabular}
\end{table}

\begin{table}[h]
\centering\small
\caption{\textbf{Where information enters.} ``Input'' means the model receives
it; ``target'' means a loss uses it; ``build/eval'' means it is used only to
construct benchmarks or to compute metrics. Test-time model inputs are the
observation history and the action only.}
\label{tab:info}
\setlength{\tabcolsep}{4pt}
\resizebox{\linewidth}{!}{%
\begin{tabular}{@{}lccc@{}}
\toprule
\textbf{Information} & \textbf{Training} & \textbf{Test} & \textbf{Build/eval only} \\
\midrule
Observation history, action $a$ & input & input & --- \\
Reference branch $o^{\anull}$ & input (\Aug), target ($\Ld$) & --- & --- \\
Context twin $o'$ & input (\Aug), target ($\Lctx$) & --- & --- \\
Support, onset and response labels & target (from the pair) & --- & --- \\
Object count and slot alignment & architecture (slot models) & architecture & synthetic metrics \\
Ground-truth graph, object identities & --- & --- & metrics \\
Physical state, SCMs and state Push-T & input, target & input & metrics \\
Physical state, pixel settings (positions, mass, friction) & --- & --- & balancing, probes \\
\bottomrule
\end{tabular}}
\end{table}

\section{Metrics Guide}\label{app:metrics}

For each metric we give its purpose, inputs, definition, range, better
direction, interpretation, where it is used, and caveats. $n$ indexes test
examples; $\mathcal P$ is the set of responsive pairs; $\mathcal N$ is the set of
nuisance slots. Physical quantities (px in Push-T, mm in CausalWorld) are decoded
from latents by a read-out probe trained on training and IID data only.

\paragraph{Relative change and relative reduction.}
For a metric $M$ where lower is better, the \emph{relative reduction} of a method
over a baseline is
\[
\text{relative reduction}=100\cdot\frac{M_{\text{base}}-M_{\text{method}}}{M_{\text{base}}}\ \%,
\]
so a positive value means the method is better. The \emph{relative change} is
$100\,(M_{\text{method}}-M_{\text{base}})/M_{\text{base}}$, the negative of the
relative reduction. In the text we write ``$X\%$ lower'' or ``$X\%$ higher''.
Rates (success, accuracy) are compared in \emph{percentage points} (pp): a change
from $83.0\%$ to $82.3\%$ is $-0.7\pp$, which is a relative change of $-0.8\%$.
Cosine and AUROC differences are absolute differences.

\paragraph{Factual prediction error ($\downarrow$).}
\field{Purpose} how well the model predicts what actually happens.
\field{Inputs} predicted and target latents of the factual branch; for the
physical version, a probe $D$ from latents to physical state $s$.
\field{Definition}
$\mathrm{MSE}_{\text{lat}}=\frac1N\sum_n\lVert\hat z_{t+1,n}-z_{t+1,n}\rVert^2$;
$\mathrm{RMSE}_{\text{phys}}=\big(\frac1N\sum_n\lVert D(\hat z_{t+1,n})-s_{t+1,n}\rVert^2\big)^{1/2}$.
\field{Range} $[0,\infty)$.
\field{Interpretation} the ``cost'' side of the trade-offs in Secs.~\ref{sec:q3}--\ref{sec:q4}.
\field{Used in} all settings (latent MSE in LeWM Push-T; block-position RMSE in
vision Push-T and CausalWorld).
\field{Caveats} latent and physical versions can disagree: a model can have lower
latent MSE and higher physical error (App.~\ref{app:cw-results}). Good prediction
does not imply correct causal structure (Sec.~\ref{sec:q12}).

\paragraph{Effect error: effect MSE and effect RMSE ($\downarrow$).}
\field{Purpose} how accurately the model predicts the difference an action makes.
\field{Inputs} $\dzh$ and $\dz$ (Eq.~\ref{eq:delta}) for the same state.
\field{Definition}
$\mathrm{RMSE}_\Delta=\big(\frac1N\sum_n\lVert\dzh_n-\dz_n\rVert^2\big)^{1/2}$;
effect MSE is its square. \emph{Direct-effect MSE} restricts the sum to the
directly intervened slot. In code, latent MSEs (effect and prediction) also
average over the $Kd$ coordinates: they equal the square of the formula above
divided by $Kd$.
\field{Range} $[0,\infty)$.
\field{Interpretation} the quantity $\Ld$ optimises.
\field{Used in} all settings; the primary metric of the LeWM and planning studies.
\field{Caveats} if few pairs are responsive, the average is dominated by null
effects (Sec.~\ref{sec:q4}). A model with the wrong mechanism can have the lowest
effect error (Table~\ref{tab:synth}).

\paragraph{Responsive effect RMSE ($\downarrow$).}
\field{Purpose} effect accuracy only where something happened.
\field{Inputs} decoded physical effects $\Delta\hat p_n,\Delta p_n$ of the
manipulated object on responsive pairs.
\field{Definition}
$\mathrm{RMSE}_{\mathcal P}=\big(\frac{1}{|\mathcal P|}\sum_{n\in\mathcal P}\lVert\Delta\hat p_n-\Delta p_n\rVert^2\big)^{1/2}$,
$\mathcal P=\{n:\lVert\Delta p_n\rVert>2\}$ (px or mm; App.~\ref{app:benchmarks}
gives the exact rule per benchmark).
\field{Range} $[0,\infty)$, in px or mm.
\field{Interpretation} the main ``benefit'' metric in vision Push-T and CausalWorld.
\field{Caveats} depends on the threshold and on probe accuracy; with few
responsive pairs it rests on a few hundred examples.

\paragraph{Effect cosine ($\uparrow$).}
\field{Purpose} whether the predicted effect points the right way, ignoring size.
\field{Inputs} predicted and true effect vectors on responsive pairs (physical,
or latent where stated).
\field{Definition}
$\overline{\cos}=\frac{1}{|\mathcal P|}\sum_{n\in\mathcal P}\langle\Delta\hat p_n,\Delta p_n\rangle/(\lVert\Delta\hat p_n\rVert\,\lVert\Delta p_n\rVert)$.
\field{Range} $[-1,1]$; $0$ means unrelated directions.
\field{Interpretation} the ``how'' of an intervention.
\field{Used in} vision Push-T, LeWM Push-T (latent; over all pairs unless the
responsive subset is named), CausalWorld.
\field{Caveats} noisy for near-zero effects and blind to magnitude; latent cosine
can stay high while physical cosine drops.

\paragraph{Response AUROC ($\uparrow$).}
\field{Purpose} whether the model knows if an action will reach the object.
\field{Inputs} a score per pair, $\sigma_n=\lVert\Delta\hat p_n\rVert$, and the
physical responsive/non-responsive label.
\field{Definition}
$\mathrm{AUROC}=\Pr(\sigma_i>\sigma_j\mid i\text{ responsive},\,j\text{ non-responsive})$.
\field{Range} $[0,1]$; $0.5$ is chance; below $0.5$ means reversed scores.
\field{Interpretation} detection of responses, not their size or direction.
\field{Used in} vision Push-T, CausalWorld, LeWM (App.~\ref{app:negatives}).
\field{Caveats} can be inflated by how samples were drawn (sampling-only AUROC
$0.83$ before the CausalWorld rebuild). Detecting a response is not predicting it.

\paragraph{Direct-target accuracy and $F_1$ ($\uparrow$).}
\field{Purpose} whether the action is routed into the right slot.
\field{Inputs} mask $m$; true direct target $T$ (evaluation only).
\field{Definition} top-1 accuracy $\frac1N\sum_n\mathbb 1[\arg\max_{i\in\mathcal C} m_{n,i}\in T_n]$
for a candidate set $\mathcal C$: all $K$ slots (\emph{all-slot}) or the object
slots only (\emph{object-only}, the value our evaluators log);
$F_1$ of $\mathbb 1[m_i\ge1/K]$ against $\mathbb 1[i\in T]$.
\field{Range} $[0,1]$; chance top-1 is $1/K$ for all slots ($1/7$ in the hard SCM)
and $1/K_{\text{obj}}$ for objects only ($0.25$).
\field{Used in} synthetic SCMs.
\field{Caveats} object-only top-1 can look like chance when the mask puts almost
all its mass on nuisances (Table~\ref{tab:synth}: $0.243$ object-only, $0.000$
all-slot); read it with the nuisance mask. It saturates on easy benchmarks (every
variant reaches object-only top-1 of $1.0$ in the easy SCM) and needs
object-aligned slots.

\paragraph{Mask leakage: nuisance mask ($\downarrow$).}
\field{Purpose} how much of the action enters slots it should not enter.
\field{Definition} $\bar m_{\mathcal N}=\frac1N\sum_n\sum_{i\in\mathcal N}m_{n,i}$.
\field{Range} $[0,1]$.
\field{Used in} synthetic SCMs.
\field{Caveats} a clean mask does not imply clean effects: the support-trained
model has nuisance mask $8.2\times10^{-5}$ but still predicts nuisance effects.

\paragraph{Nuisance effect ($\downarrow$).}
\field{Purpose} how much the model believes the action moves variables it cannot
move.
\field{Definition} the mean per-slot RMS of the predicted effect on nuisance slots,
$E_{\mathcal N}=\frac1N\sum_n\frac{1}{|\mathcal N|}\sum_{j\in\mathcal N}\big(\frac1d\lVert\dzh^{\,j}_n\rVert^2\big)^{1/2}$.
\field{Range} $[0,\infty)$, in latent units (not squared).
\field{Used in} synthetic SCMs, state Push-T.
\field{Caveats} scale-dependent; zero for a model that predicts no effects at all,
so always read together with effect error.

\paragraph{Edge AUROC and gate means ($\uparrow$ AUROC, $\uparrow$ true gate, $\downarrow$ false gate).}
\field{Purpose} whether gates rank the edges an effect actually used above the
other slot pairs.
\field{Inputs} gates $G_{ji}$ of each test example; its onset labels
$y^{\mathrm{edge}}_{ji}$, computed from the test pair's rollouts as in training.
\field{Definition} AUROC of $\{G_{ji}\}$ against $\{y^{\mathrm{edge}}_{ji}\}$ over
all off-diagonal slot pairs of all test examples (\emph{event-level}); mean gate
over positive (``true'') and negative (``off-path'') pairs. Recovery of the
structural graph is scored separately, by thresholding the mean test gate
(structural $F_1$ below).
\field{Range} $[0,1]$; AUROC $0.5$ is uninformative.
\field{Used in} synthetic SCMs, state Push-T, pixel Push-T diagnostic
(App.~\ref{app:negatives}).
\field{Caveats} a correctly learned edge that a rollout did not use counts as a
negative; the gates see slot states, not the action, so
they cannot represent labels that differ between actions in the same state. In
state Push-T the event-level
edge precision equals the contact rate ($0.724$) and is not reported.

\paragraph{Structural, event-level and schedule-aware $F_1$ ($\uparrow$).}
\field{Purpose} agreement of the thresholded graph with a reference relation.
\field{Inputs} $\hat G=\mathbb 1[G>0.5]$ and a reference: the one-hop structural
graph, the edges active in a rollout, or reachability within the observation
interval of the schedule.
\field{Definition} $F_1(\hat G,\,\text{reference})$ over off-diagonal slot pairs.
\field{Range} $[0,1]$.
\field{Used in} synthetic SCMs (Sec.~\ref{sec:q12}, App.~\ref{app:temporal}).
\field{Caveats} the wrong reference gives misleading numbers: event-level $F_1$
is $0.667$ where structural $F_1$ is $1.0$.

\paragraph{Temporal resolvability ($\uparrow$).}
\field{Purpose} whether the order of causal events is visible in the data.
\field{Inputs} per-slot effect curves $e_i(t)=\lVert\dz^{\,i}_t\rVert$ from paired
rollouts; slots matched over time.
\field{Definition} onset $t_i=\min\{t:e_i(t)\ge\tau\}$; margin between the first
and second onsets; usable fraction of pairs with a unique earliest responder;
temporal-label $F_1$ of the onset labels against one-hop edges.
\field{Range} fractions and $F_1$ in $[0,1]$; margin in steps.
\field{Used in} temporal SCM; pixel preflight (App.~\ref{app:negatives}).
\field{Caveats} equal onsets are unresolved; results depend on $\tau$.

\paragraph{Context shift of predicted effects ($\downarrow$).}
\field{Purpose} how much a change that should not matter alters the predicted effect.
\field{Inputs} predicted effects for the same state and action under $c$ and $c'$.
\field{Definition} latent: $S_{\text{ctx}}=\frac1N\sum_n\lVert\dzh_n(c)-\dzh_n(c')\rVert$;
physical: the RMSE between the two decoded block effects, in px or mm.
\field{Range} $[0,\infty)$.
\field{Used in} vision Push-T, LeWM (App.~\ref{app:negatives}), CausalWorld.
\field{Caveats} latent and physical versions can disagree (Table~\ref{tab:vision});
the true latent effect can itself depend on the context.

\paragraph{OOD/IID ratio and gap reduction ($\downarrow$ ratio, $\uparrow$ gap reduction).}
\field{Purpose} how much a correlation reversal hurts, and how much a method helps.
\field{Definition} $\rho=\varepsilon_{\text{OOD}}/\varepsilon_{\text{IID}}$;
gap reduction $=1-(\varepsilon^{\text{new}}_{\text{OOD}}-\varepsilon^{\text{new}}_{\text{IID}})/(\varepsilon^{\text{old}}_{\text{OOD}}-\varepsilon^{\text{old}}_{\text{IID}})$.
\field{Range} $\rho\ge0$ ($1$ = no degradation).
\field{Used in} state Push-T.
\field{Caveats} $\rho$ also falls if IID error rises, so report both errors.

\paragraph{Physically weighted error ($\downarrow$).}
\field{Purpose} error measured in the latent directions that change physics.
\field{Inputs} a linear physical decoder $W_D$ (latent to physics) fitted with
privileged labels.
\field{Definition} $E_{W}=\lVert W_D(\dzh-\dz)\rVert^2$ versus
$E_{\text{raw}}=\lVert\dzh-\dz\rVert^2$.
\field{Used in} CausalWorld diagnosis (App.~\ref{app:cw-results}).
\field{Caveats} a diagnostic only; it needs privileged labels.

\paragraph{Standardised mean difference, SMD ($|\cdot|\downarrow$).}
\field{Purpose} whether responsive and non-responsive pairs differ before the
intervention.
\field{Definition} $(\bar x_{\text{resp}}-\bar x_{\text{non}})/\sqrt{(s^2_{\text{resp}}+s^2_{\text{non}})/2}$
for a pre-intervention covariate $x$.
\field{Range} real; $|\mathrm{SMD}|\le0.10$ counts as balanced.
\field{Used in} CausalWorld construction (App.~\ref{app:causalworld}).
\field{Caveats} checks means of measured covariates only.

\paragraph{Sampling-only AUROC ($\to0.5$).}
\field{Purpose} whether the label can be guessed from how a sample was drawn.
\field{Inputs} only the variables used to construct the sample (initial geometry,
action, warm-up length), never its outcome.
\field{Definition} AUROC of a classifier trained on these variables to predict the
response label.
\field{Range} $[0,1]$; the ideal is $0.5$; our threshold was $0.60$.
\field{Used in} CausalWorld construction.
\field{Caveats} detects shortcuts only through the variables it is given.

\paragraph{Planning success ($\uparrow$).}
\field{Purpose} whether the model is useful for control.
\field{Definition} $\frac1E\sum_e\mathbb 1[\text{goal reached within the budget}]$
over $E$ fixed start--goal episodes, with the official success criterion.
\field{Range} $[0,100]\%$.
\field{Used in} App.~\ref{app:planning}.
\field{Caveats} dominated by the planner and its cost; with $50$ episodes,
differences of a few points are not significant. Different episode sets give
different reference values ($94\%$, $83.0\%$, $87.8\%$).

\paragraph{Rollout ratio ($\to1$; eligible if $\le1.2$).}
\field{Purpose} how much adaptation changed the base dynamics.
\field{Definition} $r=\varepsilon_{\text{adapted}}/\varepsilon_{\text{base}}$,
the $5$-step open-loop latent error on expert trajectories.
\field{Used in} planning adapters.
\field{Caveats} one-step versions are weak proxies: one-step cosine $0.999$ with the
base model still lost $16\pp$ of planning.


\section{Benchmarks}\label{app:benchmarks}

Table~\ref{tab:benchmarks} summarises the settings; the subsections below describe
each benchmark and Fig.~\ref{fig:nuisances} shows its nuisances.

\begin{table}[h]
\centering
\caption{\textbf{Settings.} Terms are the losses of Eq.~\ref{eq:full} that are
switched on. Seeds are model seeds unless stated.}
\label{tab:benchmarks}
\resizebox{\linewidth}{!}{%
\begin{tabular}{@{}lllllll@{}}
\toprule
\textbf{Setting} & \textbf{Model input} & \textbf{Intervention vs.\ reference} & \textbf{Shift at test} & \textbf{Terms used} & \textbf{Train/val/test} & \textbf{Seeds} \\
\midrule
Synthetic SCM & $7$ slots, $d{=}16$ & impulse at a noisy contact point vs.\ none & reversed action--nuisance corr. & $\Ld,\Lsup,\Ledge,\Linv$ & $5$k/$1$k/$2$k & $1$--$3$ \\
Temporal SCM & $9$ slots, $d{=}18$ & impulse vs.\ none & observation stride $\kappa$ & $\Ld,\Lsup,\Ledge,\Linv$ & $5$k/$1.2$k/$2$k & $3$ \\
State Push-T & state + $3$ nuisance slots & action vs.\ no-op & reversed corr.\ ($\rho{=}0.95$) & $\Ld,\Ledge,\Linv,\Lctx$ & $20$k/$2.5$k/$2.5$k & $3$ \\
Vision Push-T & $4$ frozen slots, $d{=}128$ & action vs.\ no-op & reversed pattern corr.\ ($\rho{=}0.95$) & $\Ld,\Linv,\Lctx$ & $12$k/$2$k/$2.5$k & $3$ \\
LeWM Push-T & $224^2$ RGB, $K{=}1$, $d{=}192$ & $5$-step macro vs.\ zero macro & none; mass, visual (planning) & $\Ld$ & $9$k/$1$k pairs & $3$ \\
CausalWorld & $64^2$ RGB, $K{=}1$, $d{=}192$ & first joint command vs.\ mirrored one & visual, mechanism, composed & $\Ld$ ($+\Lctx$) & $6$k/$1$k/$1.5$k per split & $3$ model, $3$ data \\
\bottomrule
\end{tabular}}
\end{table}

\begin{figure}[h]
\centering
\includegraphics[width=\linewidth]{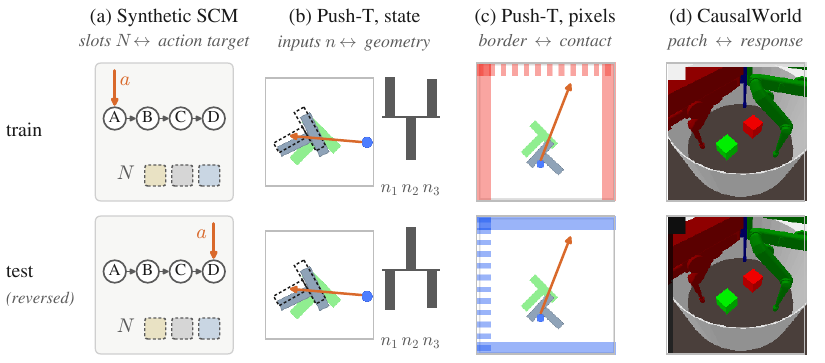}
\caption{\textbf{Nuisances and shifts in each benchmark.} In every benchmark the
nuisance is correlated with the outcome during training and the correlation is
reversed (or, for mechanism shifts, the physics is changed) at test time. The
physics of the two members of a context twin is identical by construction.}
\label{fig:nuisances}
\end{figure}

\subsection{Synthetic SCMs}\label{app:scm}
\paragraph{Hard SCM (Q1, Q2).}
Seven slots of dimension $16$: four objects and three nuisance variables. The
objects sit on a jittered ring, and a directed interaction links object $i$ to
object $i{+}1 \bmod 4$. An action is a $4$-vector: a contact point and a 2-D
impulse. It never names its target. The contact point is the target's position
moved $18\%$ of the way towards the next object on the ring, plus Gaussian noise
(s.d.\ $0.12$); the impulse points roughly towards that next object. The first
step after the action applies the impulse to the target only; the next two steps
propagate velocity along the ring (strength $0.42$, drag $0.90$). A hidden
context sets the nuisance values and biases the behaviour policy, so nuisances
co-vary with the action ($\rho=0.95$) in training; the test split reverses this.
Both branches of a pair share the hidden context and the noise. Splits:
$5{,}000$/$1{,}000$/$2{,}000$ (train/validation/test).
\paragraph{Easy SCM.}
An earlier version with three objects and one nuisance (a ``light'') driven by
the same hidden context as the action policy; the test split reverses the
correlation. It was retired because every \dojepa\ variant reached target
accuracy $1.0$ on it, so it could not tell the ingredients apart.
\paragraph{Temporal SCM (App.~\ref{app:temporal}).}
Six objects on a directed ring and three nuisances (slot dimension $18$), with
four propagation stages after a direct intervention. The edge loss sees only
coarsened rollouts: $\kappa=1$ observes stages $\{0,1,2,3,4\}$, $\kappa=2$
$\{0,2,4\}$, $\kappa=3$ $\{0,3,4\}$ and $\kappa=4$ $\{0,4\}$. The immediate
factual/reference pair is available for every $\kappa$. Splits:
$5{,}000$/$1{,}200$/$2{,}000$.

\subsection{State-based Push-T}\label{app:pusht-state}
We use the \texttt{gym-pusht} simulator \citep{chi2023diffusionpolicy} with a
5-D state (agent position, block position and angle) and three synthetic
nuisance slots. A pair branches the saved state under the factual action and a
no-op; later rollout steps are identical. An early version restored only the
body pose, not the collision and warm-start state, and replays differed by $1.6$
units; we replaced it with a seeded reset plus identical context replay, which
reproduces the branches exactly. \emph{Weak confound}: nuisances correlated with
the action in training and reversed at test. \emph{Strong confound}: nuisances
correlated with contact geometry at $\rho=0.95$, plus input-only sensor
corruption (agent noise $6$\,px, block noise $18$\,px, angle noise $0.25$\,rad,
dropout probability $0.15$). Splits: $20{,}000$/$2{,}500$/$2{,}500$, rollout
horizon $3$. The multi-step dataset uses the strong-confound protocol with
$10$-step action sequences.

\subsection{Vision Push-T}\label{app:vision-bench}
$96\times96$ RGB frames. A nuisance pattern is rendered in the image border; in
training it is correlated with contact and action geometry ($\rho=0.95$,
strength $0.45$) and at test the correlation is reversed. The pattern never
changes the physics, and the encoder saw both appearances during pretraining, so
the test is a correlation shift, not an unseen style. Each sample has three
history frames, a factual future and a no-op future from the exactly restored
state. Frozen slots: C-JEPA's Push-T VideoSAUR/DINOv2 encoder ($4$ slots
$\times128$). Splits: $12{,}000$/$2{,}000$/$2{,}500$. Before training any
predictor, we checked that the slots contain the physics: linear probes decode
agent position to $7.06$\,px, block position to $4.88$\,px and angle to
$0.026$\,rad (trivial baselines $89$\,px, $113$\,px, $1.61$\,rad), and the
nuisance moves the slots $3.16\times$ more than one step of physics. A home-made
reconstruction slot autoencoder, tried first, gave probe errors of $75$\,px and
$83^\circ$ and was discarded.

\subsection{LeWM Push-T pairs and planning corpora}\label{app:corpora}
The LeWM studies use the \texttt{stable-worldmodel} Push-T environment
\citep{swm2026} at $224\times224$. Five primitive controls form one $10$-D model
action; the reference is five exact zero relative controls. Table~\ref{tab:corpora}
lists the corpora. Labels always come from measured outcomes ($\ge2$\,px
translation or $\ge0.02$\,rad rotation of the block), never from the proposal
that generated the action.

\begin{table}[h]
\centering\small
\caption{\textbf{Push-T intervention corpora for the LeWM studies.} All use exact
same-state branching with audited replay. ``Resp.'' is the fraction of responsive
pairs.}
\label{tab:corpora}
\setlength{\tabcolsep}{4pt}
\resizebox{\linewidth}{!}{%
\begin{tabular}{@{}llrrl@{}}
\toprule
\textbf{Corpus} & \textbf{States} & \textbf{$n$} & \textbf{Resp.} & \textbf{Action source} \\
\midrule
Enriched (E) training & episode replay & $10{,}000$ & $50.0\%$ & balanced toward/away proposal \\
Outcome-unfiltered holdout & episode replay & $5{,}000$ & $5.9\%$ & i.i.d.\ from expert action marginal \\
Fresh outcome-unfiltered holdout & episode replay & $5{,}000$ & $6.5\%$ & i.i.d.\ from expert action marginal \\
Expert contiguous & expert prefix & $2{,}500$ & $50.8\%$ & contiguous expert macro \\
Expert shuffled & expert prefix & $2{,}500$ & $50.8\%$ & same macro, order shuffled \\
Expert i.i.d.\ marginal & expert prefix & $2{,}500$ & $36.2\%$ & i.i.d.\ from expert marginal \\
Natural (N) training & expert prefix & $10{,}000$ & ${\approx}51\%$ & contiguous expert macro \\
Reference contrast & expert prefix & $2{,}500$ & $52.7\%$ & factual / other expert macro / zero \\
\bottomrule
\end{tabular}}
\end{table}

The enriched corpus proposes actions towards or away from the block to reach a
$50/50$ responsive split. Because this proposal is direction-biased, the corpus
is stratified, not representative; App.~\ref{app:planning} measures what this
costs. The two ``outcome-unfiltered'' holdouts reject nothing on the basis of
outcome, but each primitive action is drawn independently from the expert action
marginal, which removes temporal structure. We call them outcome-unfiltered
rather than natural for this reason.

\subsection{CausalWorld pushing}\label{app:causalworld}
The CausalWorld pushing task \citep{ahmed2021causalworld}: a three-finger robot
pushes a block. Observations are $64\times64$ RGB frames (upsampled to $224$ for
the encoder) at steps $0,4,8,12$ of a $12$-step window; four primitive $9$-D
joint commands form one $36$-D model action. A pair shares the snapshot, a
warm-up and all commands except the first: the factual and reference first
commands are mirrored about the current joint position. The label is the maximum
difference in block displacement over the window: responsive if $\ge2$\,mm,
non-responsive if $\le0.5$\,mm, and rejected in between. Physical state is
stored for evaluation only; models see RGB and actions. Splits and shifts:
\begin{center}\small
\begin{tabular}{@{}llll@{}}
\toprule
Split & Visual shortcut & Physics (mass, friction) & Context--effect correlation \\
\midrule
Train / IID val & present & Space A & $+0.88$ to $+0.90$ \\
OOD visual & reversed & Space A & $\approx-0.90$ \\
OOD mechanism & independent & Space B (heavier) & $0.0$ \\
OOD composed & reversed & Space B & $\approx-0.90$ \\
\bottomrule
\end{tabular}
\end{center}
Block mass is in $[0.015,0.044]$ in Space A and $[0.046,0.099]$ in Space B
(simulator units). The mechanism split changes the physics and also makes the
visual context independent of the outcome, so it is not a physics-only shift. Splits: $6{,}000$ train, $1{,}000$ IID validation and $1{,}500$
per OOD split. The benchmark has a single manipulated block, so it does not test
direct-support recovery.

\paragraph{The label leak and the rebuild.}
The first version sampled responsive and non-responsive pairs separately: the
collector accepted or rejected candidates depending on the requested class.
Pre-intervention geometry then predicted the label: a classifier on sampling
variables alone reached AUROC $0.834$, and the standardised mean differences of
pre-contact distance were $-1.01$ to $-0.86$\pid{P20.3--P20.4}. Results obtained
on that version\pid{P20-CW--P20.4} are diagnostic only and are not used for any
claim. The rebuilt version samples geometry and action first, runs both branches,
and assigns the label from the outcome. It balances the classes by exact warm-up
length, $1$-mm bins of pre-contact distance and fingertip distances, with a
$12\times$ candidate pool, and it assigns the context only after balancing.
Gates: $|\mathrm{SMD}|\le0.10$ for all five geometry variables and sampling-only
AUROC $\le0.60$. The rebuilt benchmark has sampling-only AUROC $0.500$--$0.501$
and all $|\mathrm{SMD}|$ near $10^{-6}$\pid{P20.5}. The three independent
instances of Table~\ref{tab:cw} were generated with the same protocol from
dataset seeds $17$, $27$ and $37$, and all passed the same gates\pid{P21}.

\FloatBarrier
\section{Training and Evaluation Details}\label{app:training}

Table~\ref{tab:hparams} lists the settings used by the training scripts. All
models use AdamW. 

\begin{table}[h]
\centering\footnotesize
\caption{\textbf{Training settings.} ``Sel.'' is the checkpoint-selection rule;
all selection uses IID validation data. lr: learning rate; wd: weight decay.}
\label{tab:hparams}
\setlength{\tabcolsep}{3pt}
\begin{tabularx}{\linewidth}{@{}l>{\raggedright\arraybackslash}X@{}}
\toprule
\textbf{Setting} & \textbf{Settings} \\
\midrule
Hard SCM &
Transformer predictor (width $64$, depth $2$, $4$ heads), history $3$, $2$
propagation steps ($\gamma=0.55$); $25$ epochs, batch $128$, lr $3\times10^{-4}$,
wd $10^{-4}$; $\lambda_\Delta=5$, $\lambda_{\mathrm{sup}}=2$, mask temperature
$1.0$ (sparse-mask rows without $\Lsup$) or $0.7$ (with $\Lsup$, and all gated
models), entropy regulariser $0.02$ (a cardinality term, also $0.02$, is zero
under the softmax mask), no-op mask penalty $0.05$; every model is trained and
evaluated with one random history slot masked (masked-history loss $0.25$); edge
runs add $\lambda_{\mathrm{inv}}=5$ (threshold $0.05$), $\lambda_{\mathrm{edge}}=5$,
gate invariance $5$, gate $\ell_1$ $0.02$, onset threshold $0.08$, gate
temperature $0.4$. Sel.: lowest IID prediction MSE. Seeds $7$ (ablations),
$7/17/27$ (gated model). \\
Easy SCM & Transformer predictor (width $64$, depth $2$, $4$ heads); $25$ epochs,
batch $128$, lr $3\times10^{-4}$, wd $10^{-4}$; $\lambda_\Delta=5$,
$\lambda_{\mathrm{sup}}=2$, mask temperature $0.7$, entropy regulariser $0.02$;
$5{,}000$/$1{,}000$/$2{,}000$ samples. Sel.: lowest IID prediction MSE. Seeds
$7/17/27/37/47$. \\
Temporal SCM &
Width $64$, depth $2$, $4$ propagation steps; $25$ epochs, batch $128$, lr
$3\times10^{-4}$, wd $10^{-4}$; $\lambda_\Delta=5$, $\lambda_{\mathrm{sup}}=2$,
$\lambda_{\mathrm{inv}}=5$, $\lambda_{\mathrm{edge}}=5$. Sel.: lowest IID
prediction MSE. Seeds $7/17/27$. \\
State Push-T &
Width $96$, depth $3$, $4$ heads; gate temperature $0.7$; $30$ epochs, batch
$256$, lr $3\times10^{-4}$, wd $10^{-4}$; reference-branch weight $0.5$,
$\lambda_\Delta=5$, reward head $1$, $\lambda_{\mathrm{inv}}=5$ (threshold
$0.03$), $\lambda_{\mathrm{edge}}=2$, gate invariance $2$, gate $\ell_1$ $0.01$,
$\lambda_{\mathrm{ctx}}=0.25$ (context twins by shuffling nuisances within a
batch; evaluated by sign flip). Sel.: IID prediction and effect error (context
not used). Seeds $7/17/27$. \\
Vision Push-T &
Frozen slots; transformer predictor (width $256$, depth $6$, $8$ heads); $40$
epochs, batch $512$, lr $3\times10^{-4}$, wd $10^{-5}$; $\lambda_\Delta=4$,
$\lambda_{\mathrm{inv}}=1$ (soft weights, temperature $0.35$), context-twin
prediction weight $0.5$, history loss $1$, $\lambda_{\mathrm{ctx}}=0.5$
(sensitivity run $0.25$).
Sel.: IID validation. Seeds $7/17/27$. \\
LeWM Push-T (\Paired\ vs.\ \Aug) &
Released LeWM checkpoint (ViT-tiny/14 encoder, $192$-D latent, history $3$,
$6$-layer predictor); $10$ epochs, batch $64$, lr $10^{-5}$, wd $10^{-3}$,
SIGReg $0.09$, reference-branch weight $1.0$, $\lambda_\Delta=0.5$, gradient
clipping $1.0$, $10\%$ validation split, expert replay in every update. Seeds $7$
(development), $17/27/37$ (confirmation). Sel.: lowest validation paired-prediction
loss, the same rule for \Aug\ and \Paired\ (no effect or planning metric). \\
Planning adapters &
Rank-$16$ bias-free adapter ($\alpha=16$) on the frozen LeWM predictor; $10$
epochs, batch $64$, lr $10^{-4}$, wd $0$; reference-branch weight $1.0$,
$\lambda_\Delta=0.5$ (\Paired\ only), $5$-step rollout preservation weight $10$.
Sel.: lowest held-out paired prediction MSE among epochs $\ge1$ whose $5$-step
rollout error is at most $1.2\times$ the base model's; effect and planning
metrics are not used. Seeds $7$, $17/27/37$ (enriched), $7/17/27$ (natural). \\
LeWM, full objective &
Official LeWM training recipe (lr $5\times10^{-5}$, wd $10^{-3}$, bf16, SIGReg
$0.09$, frameskip $5$, warm-up cosine schedule), four branches flattened into
one batch of $128$; $\alpha=1$, $\beta=0.5$; update budget matched to standard
LeWM. Seed $7$. \\
CausalWorld &
LeWM-class encoder and predictor, $192$-D latent, identical initialisation for
all conditions; lr $5\times10^{-5}$, wd $10^{-3}$, SIGReg $0.09$, bf16; batch
$128$ ($32$ pairs $\times$ four branches for paired conditions); $40{,}000$
updates with a cosine schedule ($5{,}000$ in the first run, $10$k and $20$k in
the budget study); validation every $250$ updates; $\alpha=1$; $\beta=0.25$ for
\dojepa. Sel.: lowest IID latent factual MSE; a frozen causal-aware selector was
also tested. Model seeds $47/57/67$; model seed $107$ with dataset seeds
$17/27/37$. \\
\bottomrule
\end{tabularx}
\end{table}

\paragraph{Evaluation and aggregation.}
Metrics are computed on held-out splits after checkpoint selection. Seeds are
aggregated by the mean; we report the standard deviation or the per-seed values.
Trainers derive their train/validation split from the run seed, so effect metrics
are always computed within seed; evaluating a model on another seed's split would
expose up to $90\%$ of its training data.

\paragraph{Hardware.}
All experiments ran on NVIDIA A100 40 GB GPUs, one GPU per run. 

\FloatBarrier
\section{Additional Results}\label{app:results}

\subsection{Synthetic SCMs}\label{app:synthetic}

\begin{table}[h]
\centering\small
\caption{\textbf{Easy SCM}, reversed-correlation test split, mean over five seeds.
Every mask-based variant reaches object-only target top-1 of $1.0$. ``Nuis.\
effect'' is the RMS of the predicted effect on the nuisance slot. \Mask\ with $\Ld$
has the lowest effect MSE but $282\times$ more predicted nuisance effect than the
full model. Best in bold, second best underlined.}
\label{tab:easy}
\setlength{\tabcolsep}{4pt}
\begin{tabular}{@{}lcccc@{}}
\toprule
& Pred.\ MSE\,$\downarrow$ & Effect MSE\,$\downarrow$ & Nuis.\ effect\,$\downarrow$ & Nuis.\ mask\,$\downarrow$ \\
\midrule
\Mask\ (global action) & $0.00612$ & $0.00308$ & $0.0875$ & --- \\
\Mask\ $+\,\Ld$ & $0.00280$ & $\mathbf{0.000084}$ & $0.0110$ & --- \\
Sparse mask & $0.00446$ & $0.00258$ & $0.00606$ & $0.0257$ \\
Sparse mask $+\,\Ld$ & $\underline{0.00279}$ & $0.000261$ & $\underline{0.00240}$ & $\underline{0.0112}$ \\
Sparse mask $+\,\Ld+\Lsup$ & $\mathbf{0.00276}$ & $\underline{0.000232}$ & $\mathbf{0.000039}$ & $\mathbf{0.000197}$ \\
\bottomrule
\end{tabular}
\end{table}

\begin{table}[h]
\centering
\caption{\textbf{Where does the effect travel?} Synthetic SCM, reversed-correlation
test split; seed 7 except the last row (mean over seeds 7/17/27). Every row also
uses $\Ld+\Lsup$ and the effect part of $\Linv$. Edge AUROC and the true-edge
gate are scored against the test pairs' onset labels. ``Off-path gate'': mean gate
on slot pairs that are not on the rollout's propagation path (this includes
correctly learned but unused ring edges); ``nuis.\ in-gate'': mean gate into
nuisance slots; ``nuis.\ effect'': as in Table~\ref{tab:synth}.}
\label{tab:synth-edge}
\setlength{\tabcolsep}{3.5pt}
\resizebox{\linewidth}{!}{%
\begin{tabular}{@{}lcccccc@{}}
\toprule
& Edge AUROC\,$\uparrow$ & True-edge gate\,$\uparrow$ & Off-path gate\,$\downarrow$ & Nuis.\ in-gate\,$\downarrow$ & Nuis.\ effect\,$\downarrow$ & Direct-effect MSE\,$\downarrow$ \\
\midrule
Gates without $\Ledge$            & $0.624$ & $0.974$ & $0.465$ & $1.0\times10^{-3}$ & $5.6\times10^{-5}$ & $0.0129$ \\
$+\,\Ledge$ (onset labels)        & $0.975$ & $0.955$ & $0.053$ & $2.7\times10^{-4}$ & $3.9\times10^{-5}$ & $0.0134$ \\
$+$ gate term of $\Linv$          & $0.975$ & $0.947$ & $0.052$ & $1.0\times10^{-4}$ & $2.0\times10^{-5}$ & $0.0135$ \\
same, 3 seeds                     & $0.9751{\pm}0.0002$ & $0.950$ & $0.051$ & $8.9\times10^{-5}$ & $1.9\times10^{-5}$ & $0.0137$ \\
\bottomrule
\end{tabular}}
\end{table}

\paragraph{Evaluation protocol.} Our evaluators log object-only top-1; we
reloaded the saved checkpoints of Table~\ref{tab:synth} and recomputed top-1 over
all slots with the same test split and masking (object-only values are reproduced
exactly). The sparse-mask rows put almost all mask mass on nuisances ($1.000$), so
their all-slot top-1 is $0.000$; the support-trained row and the three gated
models score the same under both definitions ($0.9995$--$1.000$). Every row is
evaluated with one random history slot masked, as in training. C-JEPA itself
predicts from the full history at test time, and the action path of the
mask-based rows reads the unmasked last frame, which \Mask\ does not see. The
sparse-mask rows without $\Lsup$ use mask temperature $1.0$ and the row with
$\Lsup$ uses $0.7$; we did not run a temperature-matched control, so the jump in
Table~\ref{tab:synth} reflects both changes.

\paragraph{Easy SCM and gated models.}
Table~\ref{tab:easy} gives the easy SCM\pid{W1-2}. Each added component removes
about an order of magnitude of predicted nuisance effect; the support loss
removes the last factor of $61$. In the hard SCM, the three seeds of the gated
model give prediction MSE $0.0041$, direct-effect MSE $0.0137$, target $F_1$
$0.9995\pm0.0003$, target mask $0.994$, descendant mask $0.003$ and nuisance
in-gate $8.9\times10^{-5}$\pid{hard-edge}. Structural recovery per seed (mean
test gate thresholded at $0.5$): precision $=$ recall $=1$ on seeds $7$, $17$ and
$27$; mean ring-edge gate $0.947$, $0.951$, $0.953$; mean gate on other pairs
$0.0053$, $0.0025$, $0.0018$; largest non-edge gate $0.059$.

\subsection{Temporal resolvability}\label{app:temporal}

\begin{table}[h]
\centering\small
\caption{\textbf{Temporal SCM}, reversed-correlation test split, mean over seeds
7/17/27. $\kappa$ causal stages are hidden in one observed step. ``Onset-label
$F_1$'' scores the training labels themselves as one-hop edges. The
schedule-aware relation is reachability within the observation interval.}
\label{tab:temporal}
\setlength{\tabcolsep}{5pt}
\begin{tabular}{@{}lcccc@{}}
\toprule
$\kappa$ & $1$ & $2$ & $3$ & $4$ \\
\midrule
Direct-target $F_1$ & $0.997$ & $0.997$ & $0.997$ & $0.997$ \\
Descendant-effect $F_1$ & $0.947$ & $0.889$ & $0.889$ & $0.889$ \\
One-hop structural $F_1$ & $\mathbf{0.999}$ & $0.549$ & $0.502$ & $0.337$ \\
One-hop structural AUROC & $1.000$ & $0.910$ & $0.884$ & $0.843$ \\
Onset-label $F_1$ & $1.00$ & $0.40$ & $0.40$ & $0.25$ \\
Schedule-aware $F_1$ & $0.999$ & $0.935$ & $0.992$ & $0.894$ \\
Schedule-aware AUROC & $1.000$ & $1.000$ & $1.000$ & $0.771$ \\
Prediction MSE & $0.0105$ & $0.0103$ & $0.0105$ & $0.0104$ \\
\bottomrule
\end{tabular}
\end{table}

Table~\ref{tab:temporal} gives the full results\pid{P13, P13.1}. The
implied reachability horizon of the learned graph matches the schedule's
effective horizon $(1,3,3,4)$ on every seed. At $\kappa=4$ the learned graph is
too dense (edge density $1.0$ against $0.8$ for the schedule relation), which
lowers its AUROC. The multi-timescale model has one gate head per schedule;
coarse heads are tied to reachability of the one-hop head by a differentiable
max--min path composition, $G^{(\kappa)}\approx\mathrm{Reach}_{h_\kappa}(G^{(1)})$.
It reaches one-hop $F_1=0.977$ at $\kappa=1$ and schedule $F_1$
$0.955/0.977/0.909$ at $\kappa=2/3/4$ (3 seeds). A single shared gate trained on
all schedules reaches one-hop $F_1\approx0.509$. Tying the heads improves
cross-scale consistency by $+0.083$, $+0.070$ and $+0.022$ $F_1$ over
independent heads, at a small one-hop cost ($0.977$ vs.\ $0.999$)\pid{P13.2}.

\subsection{State-based Push-T}\label{app:pusht-results}

\paragraph{Causal routing and gates.}
Table~\ref{tab:pusht-state} gives the four models of Sec.~\ref{sec:q12}\pid{P1}.

\begin{table}[h]
\centering\small
\caption{\textbf{State-based Push-T}, reversed-correlation split, mean of seeds
7/17/27. ``Agent pos.'' is the agent-position MSE in px$^2$; ``nuis.\ effect'' is
the mean per-slot RMS of the predicted effect on nuisance slots; edge AUROC is
scored against the test pairs' onset labels. The agent is the known direct target
in all rows. Arrows give the better direction. Best in bold, second best underlined.}
\label{tab:pusht-state}
\setlength{\tabcolsep}{4pt}
\resizebox{\linewidth}{!}{%
\begin{tabular}{@{}lccccc@{}}
\toprule
& Pred.\ MSE\,$\downarrow$ & Agent pos.\,$\downarrow$ & Nuis.\ effect\,$\downarrow$ & Edge AUROC\,$\uparrow$ & Nuis.\ in-gate\,$\downarrow$ \\
\midrule
\Obs\ (global action)                       & $2.77\times10^{-4}$ & $5.49$ & $1.9\times10^{-3}$ & --- & --- \\
Global action, $+\Ld+\Linv$                 & $2.78\times10^{-4}$ & $3.16$ & $4.8\times10^{-4}$ & --- & --- \\
Action routed to agent, no gates            & $\underline{2.75\times10^{-4}}$ & $\mathbf{0.93}$ & $\underline{2.2\times10^{-4}}$ & $0.500$ & $1.000$ \\
Routed, $+$ gates, $\Ledge$, gate invariance & $\mathbf{2.72\times10^{-4}}$ & $\underline{0.95}$ & $\mathbf{2.1\times10^{-7}}$ & $\mathbf{0.994}$ & $\mathbf{3.1\times10^{-5}}$ \\
\bottomrule
\end{tabular}}
\end{table}

\begin{figure}[h]
\centering
\includegraphics[width=\linewidth]{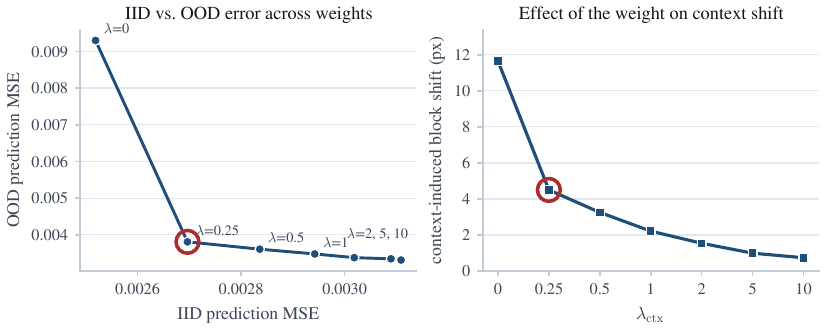}
\caption{\textbf{Context weight in state Push-T}, strong confound. Left: IID and
OOD prediction error against $\lambda_{\mathrm{ctx}}$ (seed 7); the circled point
(both panels) is $\lambda_{\mathrm{ctx}}=0.25$, chosen as the point on the IID--OOD trade-off
curve closest to the ideal, so this seed's OOD split informed the choice. Right:
the context-induced block displacement falls monotonically with
$\lambda_{\mathrm{ctx}}$.}
\label{fig:context}
\end{figure}

\paragraph{Context weight.}
Per seed, $\lambda_{\mathrm{ctx}}=0.25$ against $\lambda_{\mathrm{ctx}}=0$
changes the IID error by $+7.0\%$, $+4.2\%$ and $+3.5\%$, the OOD error by
$-59.0\%$, $-43.9\%$ and $-55.2\%$, and the OOD/IID ratio from $3.69$, $2.72$,
$2.95$ to $1.41$, $1.46$, $1.28$ (seeds $7$, $17$, $27$)\pid{P4}. Edge AUROC is
$0.994$ in every run. The first context run used $\lambda_{\mathrm{ctx}}=5$,
which closed $97\%$ of the gap at a $19\%$ IID cost\pid{P3}.

\subsection{Vision Push-T}\label{app:vision}

\paragraph{Main comparison.}
Vision Push-T models share frozen VideoSAUR/DINOv2 slots
\citep{zadaianchuk2023videosaur,oquab2024dinov2} from C-JEPA's Push-T
encoder\pid{P12.1}; a pattern rendered in the image border is correlated with
contact geometry ($\rho=0.95$) in training and reversed at test. We compare
\Mask, \Aug\ and \dojepa\ ($\Ld+\Linv+\Lctx$; $\lambda_{\mathrm{ctx}}=0.5$ chosen
on IID data\pid{P12.4}; Table~\ref{tab:vision})\pid{P12.5}.
\dojepa\ beats C-JEPA-style masking on every metric and every seed: factual error
$39.0\to27.2$\,px ($30\%$ lower), responsive effect error $20.3\to17.9$\,px ($12\%$
lower), cosine $0.60\to0.74$, response AUROC $0.49\to0.69$ (masking is at chance),
and context shift of predicted effects $18\%$ lower in latent space and $31\%$ lower
in pixels. Against \Aug, which sees the same branches, the result is mixed:
\dojepa\ detects better whether the block will respond (AUROC $+0.025$, $3/3$
seeds), and its predicted effects vary $20\%$ less with the context in latent space
($3/3$); but it has $9.3\%$ higher factual error and $6.5\%$ higher responsive
effect error than \Aug, and in pixels \Aug's predicted effects are less
context-sensitive ($12.6$ vs.\ $13.7$\,px; $0/3$ seeds favour \dojepa).

\begin{table}[h]
\centering\small
\caption{\textbf{Vision Push-T}, frozen VideoSAUR slots, reversed-pattern OOD
split, mean$\pm$sd over seeds 7/17/27. Block effects are decoded to pixels by a
context-balanced probe trained on IID data. ``Latent ctx.'' is the context shift
of predicted effects in slot space; ``Phys.\ ctx.'' is the same shift decoded to
pixels. Arrows give the better direction. Best in bold, second best underlined.}
\label{tab:vision}
\setlength{\tabcolsep}{3.5pt}
\resizebox{\linewidth}{!}{%
\begin{tabular}{@{}lcccccc@{}}
\toprule
& Factual (px)\,$\downarrow$ & Resp.\ effect (px)\,$\downarrow$ & Cosine\,$\uparrow$ & AUROC\,$\uparrow$ & Latent ctx.\,$\downarrow$ & Phys.\ ctx.\ (px)\,$\downarrow$ \\
\midrule
\Mask\ (C-JEPA-style) & $39.02\pm1.24$ & $20.25\pm0.96$ & $0.600\pm0.108$ & $0.485\pm0.020$ & $\underline{1.003\pm0.054}$ & $19.91\pm1.50$ \\
\Aug                 & $\mathbf{24.88\pm0.58}$ & $\mathbf{16.79\pm0.11}$ & $\mathbf{0.763\pm0.010}$ & $\underline{0.666\pm0.008}$ & $1.024\pm0.014$ & $\mathbf{12.59\pm0.51}$ \\
\dojepa              & $\underline{27.19\pm0.16}$ & $\underline{17.89\pm0.20}$ & $\underline{0.744\pm0.013}$ & $\mathbf{0.691\pm0.008}$ & $\mathbf{0.819\pm0.019}$ & $\underline{13.74\pm0.68}$ \\
\bottomrule
\end{tabular}}
\end{table}

\begin{table}[h]
\centering\small
\caption{\textbf{Vision Push-T screen} (seed 7, OOD). ``Paired pred.'' predicts
both branches as separate targets (an \Aug\ without context twins);
``+\,ctx.\ twins'' adds context-swapped copies as data (\Aug). \dojepa\ here uses
$\lambda_{\mathrm{ctx}}=2$. The true latent effect is itself context-sensitive
(target context shift $2.13$). Best in bold, second best underlined.}
\label{tab:vision-screen}
\setlength{\tabcolsep}{4pt}
\resizebox{\linewidth}{!}{%
\begin{tabular}{@{}lcccccc@{}}
\toprule
& Factual (px)\,$\downarrow$ & Resp.\ eff.\ (px)\,$\downarrow$ & Cos.\,$\uparrow$ & AUROC\,$\uparrow$ & Lat.\ ctx.\,$\downarrow$ & Phys.\ ctx.\ (px)\,$\downarrow$ \\
\midrule
Object-centric JEPA, no mask & $36.68$ & $21.13$ & $0.560$ & $0.481$ & $1.021$ & $20.45$ \\
\Mask & $40.41$ & $21.30$ & $0.479$ & $0.465$ & $\underline{1.008}$ & $21.22$ \\
Paired pred. & $33.23$ & $20.51$ & $0.620$ & $0.529$ & $1.027$ & $17.07$ \\
Paired pred.\ $+$ ctx.\ twins (\Aug) & $\mathbf{24.31}$ & $\mathbf{16.83}$ & $\mathbf{0.756}$ & $\underline{0.672}$ & $1.009$ & $\mathbf{12.30}$ \\
\Paired & $37.15$ & $20.62$ & $0.597$ & $0.537$ & $1.066$ & $17.43$ \\
\dojepa\ ($\lambda_{\mathrm{ctx}}=2$) & $\underline{26.78}$ & $\underline{18.68}$ & $\underline{0.727}$ & $\mathbf{0.703}$ & $\mathbf{0.515}$ & $\underline{13.71}$ \\
\bottomrule
\end{tabular}}
\end{table}

\begin{figure}[h]
\centering
\includegraphics[width=\linewidth]{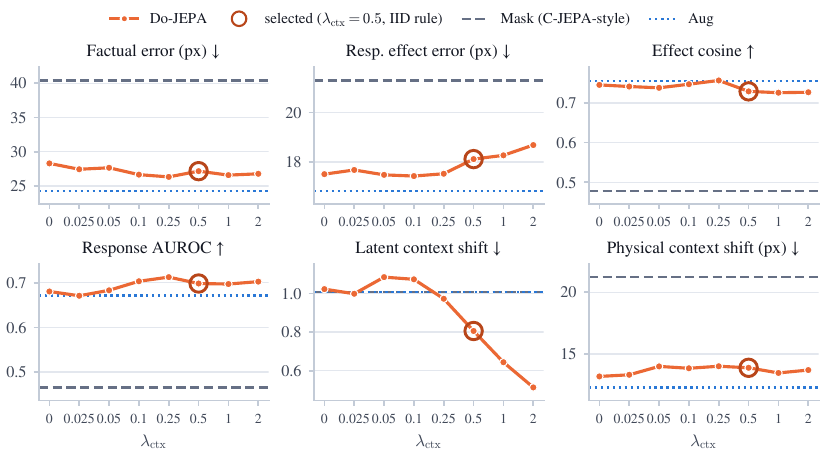}
\caption{\textbf{Context weight in vision Push-T} (seed 7, OOD). \dojepa\ across
$\lambda_{\mathrm{ctx}}$ (orange) against C-JEPA-style masking (dashed) and \Aug\
(dotted) on the same split. \dojepa\ beats masking on every metric at every weight,
and a larger weight halves the latent context shift. The circled weight was
selected on IID data as the smallest one with at least $20\%$ lower latent context
shift and at most $5\%$ higher effect RMSE, $0.02$ lower AUROC and $0.03$ lower
cosine; OOD values are shown.}
\label{fig:vision-sweep}
\end{figure}

\paragraph{Screen and weight sweep.}
Table~\ref{tab:vision-screen} gives the six-condition screen\pid{P12.2} and
Fig.~\ref{fig:vision-sweep} the weight sweep\pid{P12.4}. The three-seed results
(Table~\ref{tab:vision}) also include a $\lambda_{\mathrm{ctx}}=0.25$ run:
factual $27.17$\,px, responsive effect $17.53$\,px, cosine $0.751$, AUROC
$0.692$, latent context shift $0.964$, physical context shift
$13.76$\,px\pid{P12.5}.

\subsection{LeWM Push-T: per-seed results}\label{app:lewm-seeds}

\begin{table}[h]
\centering\small
\caption{\textbf{Effect loss on LeWM, pixel Push-T, per seed} (\Aug\ $\to$
\Paired). Left: held-out pairs from the training distribution (half responsive).
Right: fresh outcome-unfiltered holdout ($5{,}000$ pairs, $323$ responsive).
All metrics are latent, in each model's own latent space; cosines are over all
pairs. Factual columns give the relative change in factual MSE.}
\label{tab:paired-seeds}
\setlength{\tabcolsep}{3.5pt}
\resizebox{\linewidth}{!}{%
\begin{tabular}{@{}lcccccc@{}}
\toprule
& \multicolumn{3}{c}{\textbf{Held-out training-distribution pairs}} & \multicolumn{3}{c}{\textbf{Outcome-unfiltered holdout}} \\
\cmidrule(lr){2-4}\cmidrule(lr){5-7}
Seed & Effect RMSE\,$\downarrow$ & Cosine\,$\uparrow$ & Factual MSE & Effect RMSE\,$\downarrow$ & Cosine\,$\uparrow$ & Factual MSE \\
\midrule
$17$ & $0.2123\to0.1507$ & $0.669\to0.738$ & $-40.9\%$ & $0.1031\to0.0914$ & $0.177\to0.149$ & $-17.7\%$ \\
$27$ & $0.2118\to0.1512$ & $0.683\to0.744$ & $-38.4\%$ & $0.1020\to0.0906$ & $0.169\to0.145$ & $-17.3\%$ \\
$37$ & $0.2160\to0.1565$ & $0.663\to0.730$ & $-39.9\%$ & $0.1020\to0.0907$ & $0.173\to0.145$ & $-17.3\%$ \\
\midrule
Mean change & $-28.4\%$ & $+0.066$ & $-39.7\%$ & $-11.2\%$ & $-0.027$ & $-17.5\%$ \\
\bottomrule
\end{tabular}}
\end{table}

The outcome-unfiltered holdout was checked with a standalone evaluator that
reproduces the original evaluation exactly (difference $0.0$)\pid{P22B}. On its
responsive subset, the mean change is $-0.5\%$ in effect RMSE and $-0.068$ in
cosine; the overall gain comes from non-responsive pairs ($-17.3\%$). The
single-seed development run gave $-29.1\%$ effect RMSE and $+0.077$
cosine\pid{P18.1}.

\paragraph{Physical read-out.}
The metrics above live in each model's own latent space. To read them in physical
units we fit, for each model, a ridge map from its true latent paired effect
$z^{a}-z^{\anull}$ to the block displacement (px) on its own training pairs (ridge
weight chosen on a held-out tenth of them) and apply it to the model's predicted
effect $\hat z^{\,a}-\hat z^{\,\anull}$\pid{P25.0}. This post-hoc diagnostic uses
privileged labels for evaluation only (Table~\ref{tab:pusht-physical}). On held-out
pairs from the training distribution, \Paired\ lowers the responsive physical effect
error on every seed (mean $7.0\%$; $8.2\%$ over all pairs) and turns the responsive
direction from wrong to right (cosine $-0.19\to+0.23$). On the $323$ responsive
pairs of the outcome-unfiltered holdout it is $2.6$--$3.2\%$ worse (cosine
$0.18\to0.11$). Absolute accuracy is low for both models: their responsive error is
close to predicting no displacement, whereas decoding the true latent effects
reaches $6.0$--$9.5$\,px (cosine $0.71$--$0.75$), so the latent code holds more of
the effect than either predictor uses.

\paragraph{Training on natural action sequences.}
To test whether the gain follows the training intervention distribution, we
trained \Aug\ and \Paired\ with the unchanged recipe on the natural corpus
($10{,}000$ pairs with contiguous expert macros; Table~\ref{tab:corpora}) and
evaluated all twelve models on the expert-contiguous arm of the natural sequence
benchmark ($2{,}500$ pairs whose episodes are excluded from the natural training
corpus) and on the responsive pairs of the outcome-unfiltered
holdout\pid{P25.3}. The read-out is fitted per model on the natural corpus. The
design was fixed in the protocol file before the runs
(Table~\ref{tab:pusht-dist}). Trained on natural sequences, \Paired\ lowers the
responsive physical effect error on the natural test set by $13.3$--$13.6\%$ and
on the holdout by $5.3$--$5.7\%$; trained on the enriched corpus it is $6\%$
and $4$--$5\%$ worse on the same sets. The natural-trained models are also far more
accurate in absolute terms (responsive error $10$--$12$\,px against $19.7$\,px for
predicting no displacement; cosine $0.80$). The same pattern appears in the
planning transfer matrix (App.~\ref{app:planning}).

\begin{table}[h]
\centering\small
\caption{\textbf{The LeWM Push-T gain follows the training intervention
distribution.} \Paired\ vs.\ \Aug\ trained with the same recipe on the enriched or
the natural corpus; per-seed relative change (seeds 17/27/37) and mean responsive
physical cosine (\Aug\ $\to$ \Paired). Latent effect RMSE is over all pairs,
physical error over responsive pairs.}
\label{tab:pusht-dist}
\setlength{\tabcolsep}{3pt}
\resizebox{\linewidth}{!}{%
\begin{tabular}{@{}llccc@{}}
\toprule
Trained on & Tested on & Latent effect RMSE (\%) & Physical effect RMSE (\%) & Physical cosine \\
\midrule
Natural & natural test ($2{,}500$) & $-18.1/-18.4/-18.2$ & $-13.6/-13.6/-13.3$ & $0.754\to0.798$ \\
Natural & unfiltered holdout ($323$) & $-8.6/-8.1/-8.2$ & $-5.7/-5.6/-5.3$ & $0.441\to0.547$ \\
Enriched & natural test ($2{,}500$) & $-3.6/-3.1/-3.2$ & $+5.8/+5.8/+6.1$ & $0.354\to0.241$ \\
Enriched & unfiltered holdout ($323$) & $-1.0/-0.2/-0.3$ & $+3.7/+5.3/+5.2$ & $0.168\to0.101$ \\
\bottomrule
\end{tabular}}
\end{table}

\begin{table}[h]
\centering\small
\caption{\textbf{Physical read-out of the LeWM Push-T effect gain} (\Aug\ $\to$
\Paired, per seed; responsive pairs; px). ``Zero'': error of predicting no
displacement; ``decoded true'': the read-out applied to the true latent effects
(\Aug\ / \Paired).}
\label{tab:pusht-physical}
\setlength{\tabcolsep}{3.5pt}
\resizebox{\linewidth}{!}{%
\begin{tabular}{@{}lcccccc@{}}
\toprule
& \multicolumn{4}{c}{\textbf{Held-out training-distribution pairs}} & \multicolumn{2}{c}{\textbf{Outcome-unfiltered holdout}} \\
\cmidrule(lr){2-5}\cmidrule(lr){6-7}
Seed & Effect RMSE\,$\downarrow$ & Cosine\,$\uparrow$ & Zero & Decoded true & Effect RMSE\,$\downarrow$ & Cosine\,$\uparrow$ \\
\midrule
$17$ & $11.75\to11.03$ & $-0.188\to0.240$ & $10.96$ & $7.49$ / $7.24$ & $13.21\to13.56$ & $0.206\to0.124$ \\
$27$ & $11.73\to10.67$ & $-0.198\to0.237$ & $10.67$ & $6.23$ / $6.02$ & $13.29\to13.71$ & $0.177\to0.112$ \\
$37$ & $13.78\to13.00$ & $-0.170\to0.212$ & $12.63$ & $9.54$ / $9.33$ & $13.38\to13.77$ & $0.160\to0.084$ \\
\bottomrule
\end{tabular}}
\end{table}

\subsection{CausalWorld}\label{app:cw-results}

\begin{table}[h]
\centering\small
\caption{\textbf{CausalWorld fresh-seed confirmation} (model seeds 47/57/67,
$40$k updates, $\beta=0.25$, frozen selector, OOD opened once). Relative changes of
the candidate against the reference, mean over seeds. The rule for each split: at
least one benefit (AUROC $+0.02$, cosine $+0.05$ or $10\%$ lower responsive
effect error) with at most $10\%$ higher factual error, on at least $2/3$ seeds.}
\label{tab:p2012}
\setlength{\tabcolsep}{3pt}
\resizebox{\linewidth}{!}{%
\begin{tabular}{@{}llccccccc@{}}
\toprule
Hypothesis & Split & Resp.\ effect RMSE & Cosine & AUROC & Factual RMSE & Ctx.\ shift & Seeds & Result \\
\midrule
H1: \Paired\ vs.\ \Aug & visual & $11.3\%$ lower & $+0.056$ & $-0.019$ & $17.2\%$ higher & --- & $1/3$ & not met \\
 & mechanism & $22.0\%$ lower & $+0.067$ & $-0.022$ & $15.1\%$ higher & --- & $1/3$ & not met \\
 & composed & $24.3\%$ lower & $+0.059$ & $-0.012$ & $17.8\%$ higher & --- & $1/3$ & not met \\
\midrule
H2: \dojepa\ vs.\ \Paired & visual & $15.2\%$ higher & $-0.130$ & $-0.019$ & $33.3\%$ higher & $47.2\%$ lower & $0/3$ & not met \\
 & mechanism & $7.1\%$ higher & $-0.113$ & $-0.021$ & $31.7\%$ higher & $46.1\%$ lower & $0/3$ & not met \\
 & composed & $8.8\%$ higher & $-0.125$ & $-0.031$ & $32.3\%$ higher & $46.6\%$ lower & $0/3$ & not met \\
\midrule
H3: \dojepa\ vs.\ \Aug & visual & $2.2\%$ higher & $-0.075$ & $-0.038$ & $56.2\%$ higher & --- & $0/3$ & not met \\
 & mechanism & $16.5\%$ lower & $-0.046$ & $-0.043$ & $51.6\%$ higher & --- & $0/3$ & not met \\
 & composed & $17.7\%$ lower & $-0.066$ & $-0.043$ & $55.8\%$ higher & --- & $0/3$ & not met \\
\bottomrule
\end{tabular}}
\end{table}

\paragraph{Training budget and a single-seed pass.}
On the rebuilt benchmark at $10$k updates, the IID rule (factual error at most
$10\%$ higher, responsive effect error at most $5\%$ higher, cosine and AUROC at
most $0.03$ lower, context shift at least $20\%$ lower than \Paired) selected
$\beta=0.25$\pid{P20.7}. On seed 7 this model passed its OOD gate on all three
splits against \Aug, and \Paired\ alone would also have passed ($+0.10$ to $+0.11$
cosine; $9$--$11\%$ lower responsive effect error), so the pass is attributable to
pairing. This run was single-seed and not fully blind to OOD results, because OOD
numbers for $\beta=0$ and $\beta=1$ were known. Across the checkpoints of longer
runs, factual latent MSE and causal metrics were strongly anti-correlated
($|\text{Spearman}|=0.93$), so we froze a causal-aware selector for the fresh-seed
test\pid{P20.10, P20.11}.

\paragraph{Fresh-seed confirmation.}
Table~\ref{tab:p2012} gives the full result\pid{P20.12}. H1 reproduces the
benefit of Table~\ref{tab:cw} on fresh model seeds (responsive effect error
$11$--$24\%$ lower, cosine $+0.06$) with a factual cost of $15$--$18\%$, just above
the preregistered $10\%$ tolerance, so it is not confirmed (only seed $67$ passes);
Table~\ref{tab:cw-ft} shows that this cost is specific to training from scratch.
The context term (H2, H3) lowers the latent context shift by $46$--$47\%$ but costs
accuracy, as discussed in Sec.~\ref{sec:q3}. The frozen selector changed none of the
nine checkpoints.

\paragraph{Diagnosis.}
The audits of Sec.~\ref{sec:q3} use IID data only. The effect norm of the full
model falls with depth relative to \Paired: $23\%$ in the late encoder, $43\%$
after the projection and $48\%$ at the predictor output, where decoded physical
error is $29\%$ higher ($3/3$ seeds)\pid{P20.13}. The pair and context gradients
are not in conflict (cosine $\approx0.69$), but the context gradient is $2.25$
times larger and $58\%$ of it comes from the target side\pid{P20.13}. Shrinkage is not the
explanation: only $0$--$1.3\%$ of the suppressed discrepancy is radial, and
rescaling effects made the gap larger\pid{P20.15}. A predictor-side variant, which
drops the target-side term, keeps the latent pair cosine at $0.94$--$0.95$ while
the physically weighted
error rises by $76\%$, $17\%$ and $12\%$\pid{P20.16}. The linear physical decoder
has nine outputs; an SVD shows that its top two directions in the $192$-D latent
space (not two coordinates) carry about $99.8\%$ of the physical error; the
variant has $19.4\%$, $9.6\%$ and $17.5\%$ more residual energy there (the audit
criterion holds on $2/3$ seeds)\pid{P20.17}. These diagnostics rely on decoders fitted with
privileged labels.

\paragraph{IID results on three benchmark instances.}
Before opening the OOD splits of Table~\ref{tab:cw}, the IID validation results
showed the same pattern on all three instances\pid{P21A}: responsive effect
error $8.76\to7.96$\,mm, cosine $0.482\to0.537$, context shift $2.38\to0.79$,
factual error $15.07\to22.84$\,mm, AUROC $0.880\to0.817$ (\Aug\ $\to$ \Paired).

\paragraph{Where the factual cost comes from (post-hoc diagnostic).}
The Push-T runs fine-tune a pretrained model; the CausalWorld runs train from
scratch. We therefore started from each instance's converged \Aug\ checkpoint (the
one locked for Table~\ref{tab:cw}) and trained it for $2{,}000$ further updates as
\Aug\ or as \Paired\ ($\alpha=1$), with learning rate $10^{-5}$ held constant after
$100$ warm-up updates, the same data and batch composition, and no checkpoint
selection (the last checkpoint is evaluated)\pid{P25.1}. Probes were fitted as in
App.~\ref{app:causalworld}. The design was fixed in a protocol file before the runs;
it was chosen after the results of Table~\ref{tab:cw}. Table~\ref{tab:cw-ft} gives
the per-instance results: factual error stays within $-4.4$ to $+2.0\%$, response
AUROC is unchanged, and responsive effect error, cosine and latent context shift
improve on every instance and split.

We also retrained \Paired\ from scratch with the full recipe ($40$k updates,
model seed $107$, IID factual-MSE selection) and the effect weight lowered to
$\alpha=0.25$\pid{P25.2}. Against the same \Aug\ models, factual error rises by
$3.7/11.5/15.8\%$ (IID; instances 17/27/37) instead of $52\%$, responsive effect
error falls by $10.2/6.5/10.5\%$ IID and $10$--$15\%$ under the physics shifts,
cosine rises by $0.02$--$0.09$ and response AUROC falls by at most $0.044$.
Against \Paired\ with $\alpha=1$, factual error is $25$--$30\%$ lower on every
instance and split, IID responsive effect error is the same ($-3.1$ to $+1.6\%$),
and the latent context shift is about twice as large. The effect weight therefore
sets the trade-off when the representation is trained from scratch
(Table~\ref{tab:cw-regimes}).

\begin{table}[h]
\centering\small
\caption{\textbf{CausalWorld, fine-tuning a trained \Aug\ model with the effect
loss} ($2{,}000$ updates; \Paired\ vs.\ \Aug\ fine-tuned identically). Values are
per benchmark instance (dataset seeds 17/27/37): relative changes in \% for factual
error, responsive effect error and latent context shift; differences for cosine and
AUROC. $^{*}$Also decorrelates the visual context.}
\label{tab:cw-ft}
\setlength{\tabcolsep}{3pt}
\resizebox{\linewidth}{!}{%
\begin{tabular}{@{}lccccc@{}}
\toprule
Split & Factual RMSE (\%) & Resp.\ effect RMSE (\%) & Effect cosine & Response AUROC & Latent ctx.\ shift (\%) \\
\midrule
IID & $-0.5/+1.7/-3.0$ & $-6.1/-8.9/-9.5$ & $+0.044/+0.051/+0.091$ & $-0.006/-0.004/-0.001$ & $-24.2/-24.7/-28.3$ \\
Visual & $-0.2/+2.0/-4.4$ & $-7.5/-8.7/-10.6$ & $+0.060/+0.050/+0.080$ & $+0.009/+0.006/+0.010$ & $-24.8/-24.4/-28.1$ \\
Mechanism$^{*}$ & $+0.1/+1.6/-4.0$ & $-11.5/-12.5/-16.2$ & $+0.042/+0.041/+0.085$ & $+0.003/+0.006/+0.004$ & $-25.7/-25.4/-28.4$ \\
Composed & $-0.8/+1.9/-4.2$ & $-11.4/-12.7/-15.7$ & $+0.042/+0.064/+0.079$ & $+0.011/+0.015/+0.011$ & $-25.5/-24.9/-28.5$ \\
\bottomrule
\end{tabular}}
\end{table}

\FloatBarrier
\section{Planning Study Details}\label{app:planning}

\paragraph{Can paired supervision be added to a pretrained planner?}
Yes, with a centered adapter, but it does not improve planning; planning success
never selects a checkpoint. Fine-tuning the released LeWM Push-T model
\citep{maes2026lewm,swm2026} with $\Ld$ lowers effect RMSE by $29.1\%$ relative to
\Aug\ but drops planning success from $94\%$ to $54\%$ (\Aug: $58\%$)\pid{P18.1};
the damage is in the predictor's dynamics, not in the representation\pid{P18.2}.
A centered action-effect adapter (Eq.~\ref{eq:adapter}) adds
$g_\phi(h_a)-g_\phi(h_{a\leftarrow\anull})$ to the frozen predictor, so the no-op
prediction cannot change. With a $5$-step rollout-preservation loss and a
checkpoint rule that caps rollout drift, it keeps $94\%$ success; these three
parts were introduced together, so centering alone is not shown to suffice.
\Paired\ lowers effect RMSE by $5.2$--$5.8\%$ on four seeds\pid{P18.5, P18.6}, and
at $n=600$ episodes the adapted planners are equivalent to the base planner within
$\pm3\pp$\pid{P19.0, P19.1}. As in end-to-end training (App.~\ref{app:lewm-seeds}),
the gain follows the training intervention distribution: adapters trained on
natural sequences gain $2.47\%$ on natural actions, against $0.62\%$ for adapters
trained on the enriched corpus\pid{P19.1}. The rest of this section gives the
details.

\begin{table}[h]
\centering\small
\caption{\textbf{Ways of adding paired supervision to LeWM} (seed 7, $50$-episode
planning; effect metrics on each configuration's own validation pairs, so compare
within a row group only). ``Adapted'' says whether the selected checkpoint differs
from the base model.}
\label{tab:ladder}
\setlength{\tabcolsep}{4pt}
\begin{tabular}{@{}llcccc@{}}
\toprule
Configuration & Model & Adapted & Effect RMSE\,$\downarrow$ & Cosine\,$\uparrow$ & Planning\,$\uparrow$ \\
\midrule
Base LeWM & --- & --- & --- & --- & $94\%$ \\
\midrule
Full fine-tuning & \Aug & yes & $0.2145$ & $0.659$ & $58\%$ \\
 & \Paired & yes & $\mathbf{0.1522}$ & $\mathbf{0.736}$ & $54\%$ \\
\midrule
Expert data only & --- & yes & $0.6588$ & $0.428$ & $\mathbf{94\%}$ \\
Frozen representation & \Aug & yes & $0.2998$ & $0.810$ & $36\%$ \\
 & \Paired & yes & $0.2889$ & $0.863$ & $36\%$ \\
\midrule
Residual adapter & \Aug & yes & $0.6503$ & $0.443$ & $78\%$ \\
 & \Paired & \textbf{no} & $0.6580$ & $0.441$ & ($94\%$) \\
$+$ preservation sweep & \Aug & yes & $0.6382$ & $0.447$ & $84\%$ \\
 & \Paired & yes & $0.5813$ & $0.534$ & $78\%$ \\
\midrule
Centered action-effect adapter & \Aug & yes & $0.6014$ & $0.505$ & $\mathbf{94\%}$ \\
 & \Paired & yes & $\mathbf{0.5701}$ & $\mathbf{0.542}$ & $\mathbf{94\%}$ \\
\bottomrule
\end{tabular}
\end{table}

\paragraph{The ladder.}
Table~\ref{tab:ladder} gives every configuration\pid{P18.1--P18.5}. With full
fine-tuning, the embedding cosine to the base model falls to $0.689$, and the
Pearson correlation of squared latent distances to $0.553$. Freezing encoder and
projector makes the embeddings bitwise identical (embedding MSE $0$, cosine $1$,
distance correlation $1$) and still gives $36\%$. In the residual-adapter row,
the \Paired\ model's selector returned epoch $0$, the unchanged base model: its
epoch-$1$ checkpoint missed the $1.20\times$ preservation cap by $0.0006$. Its
$94\%$ is the base model, not an adapted one; the next configuration forbids
epoch $0$. The \Aug\ model of the preservation sweep has one-step prediction
cosine $0.9994$ with the base model and plans $10\pp$ worse.

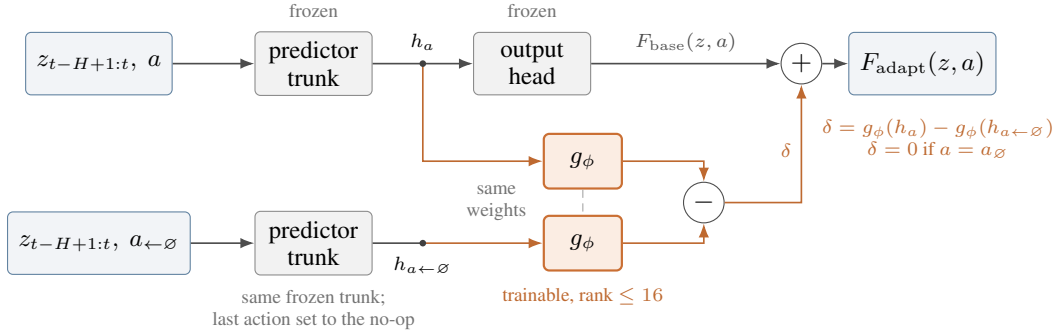
\begin{figure}[h]
\centering
\resizebox{\linewidth}{!}{%
\begin{tikzpicture}[
  font=\small,
  frozen/.style={draw=black!45, fill=black!5, rounded corners=2pt, inner sep=3.5pt, minimum height=8mm, minimum width=15mm, align=center},
  train/.style={draw=acc!90!black, fill=acc!14, rounded corners=2pt, thick, minimum height=7mm, minimum width=10mm},
  io/.style={draw=navy!70, fill=navy!6, rounded corners=2pt, inner sep=4pt, minimum height=8mm, align=center},
  op/.style={draw=black!70, fill=white, circle, inner sep=0pt, minimum size=5.2mm, font=\normalsize},
  ar/.style={-{Latex[length=1.8mm]}, black!70, semithick},
  arT/.style={-{Latex[length=1.8mm]}, acc!85!black, semithick},
  note/.style={font=\scriptsize, text=black!55, align=center}]
\node[io]     (in1)  at (0,2.3)   {$z_{t-H+1:t},\ a$};
\node[frozen] (tr1)  at (2.75,2.3) {predictor\\trunk};
\coordinate   (h1)   at (4.15,2.3);
\node[frozen] (head) at (5.55,2.3) {output\\head};
\node[op]     (plus) at (9.0,2.3) {$+$};
\node[io]     (out)  at (10.55,2.3) {$F_{\mathrm{adapt}}(z,a)$};
\node[io]     (in0)  at (0,0)     {$z_{t-H+1:t},\ a_{\leftarrow\varnothing}$};
\node[frozen] (tr0)  at (2.75,0)  {predictor\\trunk};
\coordinate   (h0)   at (4.15,0);
\node[train]  (g1)   at (6.2,1.05) {$g_\phi$};
\node[train]  (g0)   at (6.2,0)    {$g_\phi$};
\node[op]     (minus) at (7.75,0.52) {$-$};
\draw[ar] (in1) -- (tr1);
\draw[ar] (tr1.east) -- (h1) -- (head.west);
\draw[ar] (head) -- node[above=0.5pt, font=\scriptsize]{$F_{\mathrm{base}}(z,a)$} (plus);
\draw[ar] (plus) -- (out);
\draw[ar] (in0) -- (tr0);
\draw[black!70, semithick] (tr0.east) -- (h0);
\draw[arT] (h1) |- (g1);
\draw[arT] (h0) -- (g0);
\draw[arT] (g1.east) -| (minus.north);
\draw[arT] (g0.east) -| (minus.south);
\draw[arT] (minus.east) -| node[pos=0.72, left, font=\scriptsize, text=acc!85!black]{$\delta$} (plus.south);
\fill[black!75] (h1) circle (1.2pt) node[above=1.5pt, font=\scriptsize, text=black]{$h_a$};
\fill[black!75] (h0) circle (1.2pt) node[below=1.5pt, font=\scriptsize, text=black]{$h_{a\leftarrow\varnothing}$};
\draw[black!35, densely dashed] (g1.south) -- (g0.north);
\node[note, anchor=east] at (5.6,0.52) {same\\weights};
\node[note, above=0.6mm of tr1] {frozen};
\node[note, above=0.6mm of head] {frozen};
\node[note, below=0.6mm of tr0] {same frozen trunk;\\last action set to the no-op};
\node[note, below=0.8mm of g0, text=acc!85!black] {trainable, rank $\le 16$};
\node[note, anchor=north, text=acc!85!black] at (10.75,1.72) {$\delta=g_\phi(h_a)-g_\phi(h_{a\leftarrow\varnothing})$\\$\delta=0$ if $a=a_\varnothing$};
\end{tikzpicture}}
\caption{\textbf{Centered action-effect adapter} (Eq.~\ref{eq:adapter}). The frozen
predictor runs twice on the same history: with the action $a$, and with the last
action replaced by the no-op. A trainable map $g_\phi$ with shared weights is applied
to both feature vectors, and their difference $\delta$ is added to the base
prediction. $\delta$ is exactly zero for $a=\anull$, whatever the parameters $\phi$.}
\label{fig:method}
\end{figure}

\paragraph{Properties of the adapter.}
With $\Fb$ the frozen pretrained predictor, $h_a$ its internal feature under
action $a$, and $h_{a\leftarrow\anull}$ the feature with the same history and the
last action replaced by the no-op, the centered action-effect adapter is
\begin{equation}
\Fa(z,a)=\Fb(z,a)+g_\phi(h_a)-g_\phi(h_{a\leftarrow\anull}),
\label{eq:adapter}
\end{equation}
where $g_\phi(h)=W_\phi h$ is bias-free, with $W_\phi=\frac{\alpha}{r}BA$,
$A\in\R^{16\times2048}$ and $B\in\R^{192\times16}$ zero-initialised, so
$W_\phi$ starts at zero and has rank at most $16$. It is low-rank in the sense of \citet{hu2022lora} and residual in the sense
of \citet{houlsby2019adapters}; the failure it avoids, the collapse of planning
under fine-tuning, is a form of catastrophic forgetting
\citep{french1999catastrophic,kirkpatrick2017ewc}. It is trained with the paired
prediction loss, a penalty on deviation from the base model's own $5$-step
rollouts on expert data and, for \Paired, $\Ld$; a checkpoint is eligible only if
its $5$-step rollout error is at most $1.2\times$ the base model's
(Table~\ref{tab:hparams}).

\begin{proposition}[Exact no-op invariance]\label{prop:noop}
For all $z$ and all $\phi$, $\Fa(z,\anull)=\Fb(z,\anull)$.
\end{proposition}
\begin{proof}
For $a=\anull$ the two branches of Eq.~\ref{eq:adapter} receive the same input,
so $h_{\anull\leftarrow\anull}=h_{\anull}$ and the adapter terms cancel.
\end{proof}

\begin{proposition}[Linear, bounded correction]\label{prop:lin}
$g_\phi(h_a)-g_\phi(h_{a'})=W_\phi(h_a-h_{a'})$, so
$\lVert g_\phi(h_a)-g_\phi(h_{a'})\rVert\le\sigma_{\max}(W_\phi)\lVert h_a-h_{a'}\rVert$.
\end{proposition}

\begin{proposition}[Rollout drift]\label{prop:drift}
If the per-step correction has norm at most $\varepsilon$ and the base transition
is $L$-Lipschitz in its state, the $k$-step open-loop deviation is at most
$\varepsilon(L^k-1)/(L-1)$ for $L\neq1$ and $k\varepsilon$ for $L=1$.
\end{proposition}
\begin{proof}
Let $\delta_j$ be the deviation after $j$ steps. Then
$\delta_{j+1}\le L\delta_j+\varepsilon$ with $\delta_0=0$; unrolling gives
$\delta_k\le\varepsilon\sum_{j=0}^{k-1}L^j$.
\end{proof}

\begin{table}[h]
\centering\small
\caption{\textbf{Numerical checks of the adapter properties} (enriched-trained
\Paired\ adapter, seed 7, $256$ examples).}
\label{tab:props}
\setlength{\tabcolsep}{4pt}
\begin{tabular}{@{}lll@{}}
\toprule
Property & Measured & Check \\
\midrule
Prop.~\ref{prop:noop}: no-op correction & $0.0$ (bitwise) & \pass \\
Prop.~\ref{prop:lin}: linearity & max.\ error $\le2\times10^{-7}$ & \pass \\
Prop.~\ref{prop:lin}: operator bound & holds for all pairs; mean tightness $0.037$ & \pass \\
Prop.~\ref{prop:drift}: sampled $\hat L$ & $0.714$ (lower bound, not certified) & --- \\
Spectral / Frobenius norm of $W_\phi$ & $0.857$ / $1.203$ & --- \\
Rank of $W_\phi$ ($192\times2048$) & $\le16$ by construction & --- \\
Mean correction norm / mean latent norm & $0.282$ / $13.9$ & --- \\
\bottomrule
\end{tabular}
\end{table}

We introduced Prop.~\ref{prop:drift} expecting errors to grow geometrically. The
measured drift is $0.28$, $0.35$, $0.38$, $0.38$ and $0.42$ for $k=1,\dots,5$,
far below the bound evaluated with our Lipschitz estimate.
That estimate ($\hat L\approx0.71$ for the enriched-trained adapter, $0.75$ for
the natural-trained one) is the mean over $8$ batches of the largest
finite-difference ratio found with $24$ random perturbations of the map from the
$3$-frame latent history to the next frame. It is a lower bound on the Lipschitz
constant, not a certified value. Moreover, the rollout state is the whole
$3$-frame history, two frames of which are carried over unchanged at each step,
so $\hat L<1$ does not show that rollouts contract\pid{P20-theory}. A
consequence of Prop.~\ref{prop:lin}: the no-op anchor
cancels in any pairwise contrast, so the architecture can express $do(a)$
versus $do(a')$ without change. Whether the learned map transfers is tested
below.

\paragraph{Confirmation.}
Table~\ref{tab:confirm} gives the three fresh seeds\pid{P18.6}. All relative
criteria of the protocol passed; its absolute criterion ($\ge90\%$) was out of reach
because the base model itself scores $83.0\%$ on these $200$ episodes.

\begin{table}[h]
\centering\small
\caption{\textbf{Centered adapter, three fresh seeds.} Effect RMSE on held-out
enriched pairs; planning over $200$ episodes (base LeWM: $83.0\%$).}
\label{tab:confirm}
\setlength{\tabcolsep}{5pt}
\begin{tabular}{@{}lccc@{}}
\toprule
& Seed 17 & Seed 27 & Seed 37 \\
\midrule
Effect RMSE, \Aug & $0.5835$ & $0.5956$ & $0.6043$ \\
Effect RMSE, \Paired & $\mathbf{0.5496}$ & $\mathbf{0.5640}$ & $\mathbf{0.5726}$ \\
Relative reduction & $5.80\%$ & $5.31\%$ & $5.25\%$ \\
\midrule
Rollout ratio, \Paired & $1.032$ & $1.036$ & $1.033$ \\
No-op and representation unchanged & yes & yes & yes \\
\midrule
Planning, \Aug & $81.5\%$ & $84.5\%$ & $84.5\%$ \\
Planning, \Paired & $80.5\%$ & $83.5\%$ & $83.0\%$ \\
\bottomrule
\end{tabular}
\end{table}

\paragraph{Training distribution.}
Training fresh adapters on naturally occurring expert action sequences and
crossing training with test distributions gives\pid{P19.1}:
\begin{center}\small
\begin{tabular}{@{}lcccc@{}}
\toprule
& Train E & Train N & Diagonal minus off-diagonal & $95\%$ CI \\
\midrule
Test E & $\mathbf{5.43\%}$ & $0.56\%$ & $4.88\pp$ ($9.8\times$) & $[4.66,5.09]$ \\
Test N & $0.62\%$ & $\mathbf{2.47\%}$ & $1.85\pp$ ($4.0\times$) & $[1.77,1.93]$ \\
\bottomrule
\end{tabular}
\end{center}
Both contrasts exclude zero: each adapter gains most on the distribution it was
trained on. The natural-trained gain ($2.47\%$, CI $[2.37,2.56]$) is below the $5\%$
threshold fixed in the protocol but clearly above zero.

\paragraph{Planning equivalence.}
At $n=600$ episodes with per-episode pairing, the natural-trained \Paired\
adapter (seed 7) reaches $87.7\%$ against $87.8\%$ for the base planner: a paired
difference of $-0.17\pp$, exact McNemar $p=1.0$, and TOST equivalence at
$\pm3\pp$ ($90\%$ CI $[-2.20,+1.87]$)\pid{P19.1}. All six natural-trained
adapters and the enriched-trained adapters are equivalent at $\pm3\pp$ (enriched
\Paired: $-0.56\pp$, $90\%$ CI $[-2.28,+1.17]$)\pid{P19.0}. At $n=200$ the same
test passed at $\pm5\pp$ but not at $\pm3\pp$.

\FloatBarrier
\section{Boundary Conditions and Diagnostic Findings}\label{app:negatives}

Each item marks a condition under which the method, or a benchmark, does or does
not apply. We give why it was considered, what we did, what we found, the lesson,
and what it changes about our claims.

\paragraph{Benchmarks outside the method's assumptions.}
Three further benchmarks did not fit the method's assumptions. CLEVRER has no actions
or paired videos, and its counterfactual questions name objects that anonymous
VideoSAUR or SAVi slots could not be matched to (soft IoU at most
$0.075$)\pid{P14.2b, P14.2c}, so named interventions could not be applied. In PHYRE,
a support-trained model located the objects that respond to a placed ball much better
than a matched control (support AUROC $0.810$ vs.\ $0.583$) but predicted their motion
worse (responsive cosine $-0.110$ vs.\ $+0.284$)\pid{P15.2b}: support identification is
not effect prediction. In OGBench cube, exact MuJoCo restoration made the pairs exact,
but the preliminary models were too undertrained for a conclusion\pid{P17.2}.

\paragraph{Propagation from anonymous pixel slots (single-seed IID diagnostic).}
\field{Why} to test the propagation result beyond object-aligned variables.
\field{What we did} a preflight on new pixel Push-T pairs with frozen VideoSAUR
slots and no object identities: $5{,}000$ pairs whose actions differ for $10$
physics steps and then share a $20$-step continuation\pid{P23.1}. $79.2\%$ of
pairs have a unique earliest responder with median onset margin $1$ step and
$2.3\%$ alignment ambiguity; each of the four slots is the earliest responder in
about a quarter of the pairs, so roles are not tied to slot indices. We then
trained the gated architecture with and without $\Ledge$ on seed $7$ and evaluated
it on the IID test split, with pass/fail rules fixed in advance\pid{P23.2}.
\field{What we found} against the test pairs' onset labels (anonymous slots have
no ground-truth graph), edge AUROC rose from $0.823$ without $\Ledge$ to $1.000$
with it and the mean gate on other slot pairs fell from $0.305$ to $0.0002$,
while effect RMSE changed by only $2.8\%$, inside the $5\%$ margin. Direct-target
top-1 was $1.0$ on the $583$ evaluable cases. A context term added in the same
study lowered the context shift by $14$--$15\%$, short of its $20\%$ target, in
line with Sec.~\ref{sec:q3}.
\field{Effect on claims} onset supervision trains gates on anonymous visual slots
that reproduce onset order (one seed, IID data); propagation is claimed only with
object-aligned variables.

\end{document}